\documentclass{article}

\PassOptionsToPackage{numbers, compress}{natbib}

\usepackage[final]{neurips_2020}

\usepackage{bbm, amsmath, amsfonts, amssymb, amsthm, epsfig, color, enumitem, csquotes}

\DeclareMathOperator*{\argmin}{arg\,min}
\DeclareMathOperator*{\trace}{Trace}

\usepackage{microtype}
\usepackage{graphicx}
\usepackage{subfigure}
\usepackage{booktabs}

\newtheorem{theorem}{Theorem}
\newtheorem{definition}{Definition}[section]
\newtheorem{ass}{Assumptions}[section]
\newtheorem{appendixthm}{Theorem}[section]

\usepackage{hyperref}

\newcommand{\R}{\mathbb{R}}
\newcommand{\N}{\mathbb{N}}
\newcommand{\1}{\mathbf{1}}
\renewcommand{\P}{\mathbbm{P}}
\newcommand{\E}{\mathbbm{E}}
\newcommand{\F}{\mathcal{F}}
\newcommand{\T}{\mathcal{T}}
\newcommand{\eps}{\varepsilon}

\title{Numerically Solving Parametric Families of High-Dimensional Kolmogorov Partial Differential Equations via Deep Learning}

\author{
   Julius Berner\thanks{Equal contribution.} \\
   Faculty of Mathematics, University of Vienna \\
   Oskar-Morgenstern-Platz 1, 1090 Vienna, Austria\\
   \texttt{julius.berner@univie.ac.at}
   \And
   Markus Dablander$^\ast$ \\
   Mathematical Institute, University of Oxford \\
   Andrew Wiles Building, OX2 6GG, Oxford, United Kingdom \\
   \texttt{markus.dablander@maths.ox.ac.uk} \\
   \AND
   Philipp Grohs \\
Faculty of Mathematics and Research Platform DataScience@UniVienna,
University of Vienna \\ 
Oskar-Morgenstern-Platz 1, 1090 Vienna, Austria\\
RICAM, Austrian Academy of Sciences \\ Altenberger Straße 69, 4040 Linz, Austria \\
   \texttt{philipp.grohs@univie.ac.at} \\
}

\begin{document}
\maketitle

\begin{abstract}
We present a deep learning algorithm for the numerical solution of parametric families of high-dimensional linear Kolmogorov partial differential equations (PDEs). Our method is based on reformulating the numerical approximation of a whole family of Kolmogorov PDEs as a single statistical learning problem using the Feynman-Kac formula. Successful numerical experiments are presented, which empirically confirm the functionality and efficiency of our proposed algorithm in the case of heat equations and Black-Scholes option pricing models parametrized by affine-linear coefficient functions. We show that a single deep neural network trained on simulated data is capable of learning the solution functions of an entire family of PDEs on a full space-time region. Most notably, our numerical observations and theoretical results also demonstrate that the proposed method does not suffer from the curse of dimensionality, distinguishing it from almost all standard numerical methods for PDEs.
\end{abstract}
\section{Introduction}
Linear parabolic partial differential equations (PDEs) of the form
\begin{equation} \label{kol_pde_intro}
\tfrac{\partial u_\gamma }{\partial t}  = \tfrac{1}{2} \trace\big(\sigma_\gamma  [\sigma_\gamma  ]^{*}\nabla_x^2 u_\gamma \big) + \langle \mu_\gamma , \nabla_x u_\gamma \rangle, \quad
u_\gamma (x,0) = \varphi_\gamma(x),
\end{equation}
are referred to as Kolmogorov PDEs, see~\cite{hairer2015loss} for a thorough study of their mathematical properties. Throughout this paper, the functions
\begin{equation*}
\varphi_\gamma   : \mathbb{R}^d \rightarrow \mathbb{R} \quad \text{(initial condition)} \quad \text{and} \quad
\sigma_\gamma  : \mathbb{R}^d \rightarrow \mathbb{R}^{d \times d}, \quad 
\mu_\gamma  : \mathbb{R}^d \rightarrow \mathbb{R}^{d} \quad \text{(coefficient maps)}
\end{equation*}
are continuous, and are implicitly determined by a real parameter vector $\gamma \in D  $, whereby $D$ is a compact set in Euclidean space.

Kolmogorov PDEs frequently appear in applications in a broad variety of scientific disciplines, including physics and financial engineering \cite{black1973pricing, pascucci2005kolmogorov, widder1976heat}. In particular, note that the heat equation from physical modelling as well as the widely-known Black-Scholes equation from computational finance are important special cases of Equation~\eqref{kol_pde_intro}.
Typically, one is interested in finding the (viscosity) solution\footnote{Viscosity solutions are the appropriate solution concept for a wide range of PDEs~\cite{crandall1992user, hairer2015loss}. Viscosity solutions are continuous, but not necessarily differentiable.}
\begin{equation*}
    u_\gamma  : [v,w]^d \times [0,T] \rightarrow \mathbb{R}
\end{equation*}
of Equation~\eqref{kol_pde_intro} on a predefined space-time region of the form $[v,w]^{d} \times [0,T]$. In almost all cases, however, Kolmogorov PDEs cannot be solved explicitly. Furthermore, standard numerical solution algorithms for PDEs, in particular those based on a discretization of the considered domain, are known to suffer from the so-called \textit{curse of dimensionality}\footnote{The classical way to circumvent the curse of dimensionality has been the employment of stochastic Monte Carlo based methods, see e.g.~\cite{graham2013stochastic}; these techniques, however, are only suitable to approximately compute the solution $u_\gamma(x,t)$ at a single \emph{fixed} space-time point $(x,t) \in [v,w]^d \times [0,T]$, limiting their usefulness in practice.}, meaning that their computational cost grows exponentially in the dimension of the domain~\cite{ames2014numerical, seydel2006tools}. The development of new, computationally efficient methods for the numerical solution of Kolmogorov PDEs is therefore of high interest for applied scientists. 

In this work, we present a novel deep learning algorithm capable of numerically approximating the solutions $(u_\gamma )_{\gamma \in D}$ of a whole family of $\gamma$-parametrized Kolmogorov PDEs on a full space-time region. Specifically, our proposed method allows to train a single deep neural network 
\begin{equation}
\Phi\colon D   \times [v,w]^d \times [0,T] \rightarrow \mathbb{R}
\end{equation}
to approximate the \textit{parametric solution map}
\begin{align}
\label{eq:gen_sol_map}
    \bar{u} : D   \times [v,w]^d \times [0,T] \rightarrow \mathbb{R}, \quad 
    (\gamma, x, t) \mapsto \bar{u}(\gamma, x, t) := u_\gamma (x,t),
\end{align}
of a family of $\gamma$-parametrized Kolmogorov PDEs on the generalized domain $D   \times [v,w]^d \times [0,T]$. Most notably, we also theoretically investigate the associated approximation and generalization errors and rigorously show that our algorithm does not suffer from the curse of dimensionality with respect to the neural network size as well as the sample size. We emphasize that our empirical results strongly suggest that also the empirical risk minimization (ERM) algorithm, usually a variant of stochastic gradient descent, does not suffer from the curse of dimensionality but proving this is out of scope of this paper.

\subsection{PDEs and Deep Learning: Current Research and Related Work}
 \label{pde_deeplearning_relatedwork}
Interest in deep-learning based techniques for the numerical solution of PDEs has been growing rapidly in recent years~\cite{beck2019machine, han2018solving, jentzen2018proof, raissi2019physics,sirignano2018dgm, weinan2018deep, weinan2017deep}. This sharp rise in interest can partly be explained by the remarkable ability of deep neural networks to avoid incurring the curse of dimensionality when used to approximate PDE solutions in particular settings. More specifically, in some situations it has been possible to find theoretical upper bounds for the size of the required neural network architectures which do not depend exponentially on the dimension of the PDE~\cite{elbrachter2018dnn, grohs2018proof, hutzenthaler2019proof, schwab2019deep, reisinger2019rectified}. This represents a rare and crucial advantage over classical finite difference and finite element methods, all of which typically cannot be used in high dimensions due to the resulting exponential explosion of required computational costs. 

As a result of these successes, deep learning has recently been studied as a numerical solution technique for the more general group of parametric PDEs and their associated parametric solution maps~\cite{eigel2018variational, hesthaven2018non, khoo2017solving, kutyniok2019theoretical, laakmann2020efficient, schwab2019deep}. The investigation of the capabilities of deep artificial neural networks to learn parametric solution maps of specific parametrizable families of PDEs has become a new and active area of research. In this work, we provide novel theoretical and empirical results which, for the first time, demonstrate the viability of deep learning algorithms for the scalable solution of large classes of parametric Kolmogorov PDEs.

The formulation of the learning problem underlying our method is inspired by the work of Beck et al.~\cite{beck2018solving}. There it is shown how deep neural networks can be used to numerically solve a non-parametric version of Equation~\eqref{kol_pde_intro} with fixed initial condition $\varphi_\gamma$ and fixed coefficients maps $\sigma_\gamma, \mu_\gamma$ on a predefined space region $[v,w]^d$
and at a predefined time slice $T > 0$. In other words, their non-parametric method allows to approximate the function
\begin{align*} \label{solution_special_case}
    u_\gamma (\ {\cdot} \ , T) : [v,w]^d \rightarrow \mathbb{R}, \quad 
     x \mapsto u_\gamma (x,T),
\end{align*}
for fixed $\gamma\in D$ by training a deep neural network with suitable simulated data of the form
\begin{equation*}
    (X,\varphi_\gamma  (S_{\gamma,X,T})) \in [v,w]^d \times \mathbb{R}.
\end{equation*}
Here, $X$ is uniformly drawn from the spatial hypercube $[v,w]^d$ and the random vector $S_{\gamma,X,T}$ is the value of the solution process $(S_{\gamma,X,t})_{t\ge 0}$ of the stochastic differential equation (SDE)
\begin{align*} 
dS_{\gamma,X,t} = \mu_\gamma   (S_{\gamma,X,t}) dt + \sigma_\gamma  (S_{\gamma,X,t}) dB_t, \quad S_{\gamma,X,0} = X,
\end{align*}
at time $t=T$, whereby $(B_t)_{t \ge 0}$ is a standard $d$-dimensional Brownian motion.

The choice of training data is based on the following important identity, which under suitable regularity assumptions holds for all $x \in [v,w]^d$, $t \in [0,T]$, and $\gamma\in D$:
\begin{equation} \label{feynman-Kac}
    u_\gamma (x, t) = \mathbb{E}[\varphi_\gamma  (S_{\gamma,x,t} )].
\end{equation}
Equality~\eqref{feynman-Kac} is a version of the well-known \textit{Feynman-Kac formula} which establishes a link between the theory of parabolic PDEs and the theory of stochastic differential equations~\cite{hairer2015loss}. Using the Feynman-Kac formula, one can show within the mathematical framework of empirical risk minimization~\cite{cucker2002mathematical,vapnik1998statistical} that $u_\gamma ( \cdot \, , T)$ is in fact the solution of the supervised statistical learning problem defined by the predictor variable $X$, the target variable $\varphi_\gamma  (S_{\gamma,X,T} )$, and a standard quadratic loss function~\cite{beck2018solving}.

\subsection{Novel Contribution}
In this work, we introduce the first algorithm for the numerical solution of \textit{parametric} Kolmogorov PDEs on a \textit{whole} space-time region.
No previous technique has achieved this degree of generality; all former methods for parametric Kolmogorov PDEs were either only capable of computing local solutions at single space-time points of the domain using Monte Carlo based approaches or did not employ deep neural networks and were thus not able to break the curse of dimensionality. 
Our technique is made possible by constructing a suitable supervised learning problem via a nontrivial application of the Feynman-Kac formula~\eqref{feynman-Kac}, which involves random initial conditions and SDEs with random coefficients and stopping times. This reformulation of a broad class of parametric PDEs as learning problems provides a new theoretical framework to analyze the convergence behavior of deep learning algorithms. Building upon this framework, we prove theoretical guarantees for the computational performance of our technique and, to the best of our knowledge, establish the first combined approximation and generalization results for parametric PDEs. 

Note that the parametric nature of the presented algorithm opens up the novel possibility to study changes in the potentially high-dimensional solution manifold of Equation \eqref{kol_pde_intro} as its initial conditions and coefficient maps vary with $\gamma \in D$. The study of such changes is commonly referred to as \emph{sensitivity analysis}. Recall that the proposed method delivers a neural network $\Phi$ which approximates the parametric PDE solution map, i.e.\@ $\Phi \approx \bar{u}$. The partial derivatives of $\Phi$ with respect to the parameter $\gamma$, the spatial variable $x$, and the time variable $t$ can then be readily computed via automatic differentiation. Thus, the partial derivatives of $\Phi$ become computationally accessible approximations of the partial derivatives of $\bar{u}$. The partial derivatives of $\bar{u}$ in turn play an important role in a variety of widespread applications, such as in the computation of the \enquote{Greeks} associated with the Black-Scholes model from financial engineering (see Section \ref{section_black_scholes}).

Another highly relevant application area opened up by our method is the calibration of the usually unknown PDE coefficients $\sigma_\gamma, \mu_\gamma$ using real-world data. After solving a parametric PDE with our technique, one can fit $\gamma$ such that the PDE solution manifold best describes a real data set and additionally apply uncertainty quantification techniques if $\gamma$ is modelled as a random variable.

Finally, we establish a new architecture and compare different learning schemes to provide future researchers with a robust framework for parametric PDEs, which are inherently less stable than their simpler non-parametric counterparts. Further, this work is complemented by an extendable implementation with the possibility of distributed training and hyperparameter optimization for the special use-cases of other researchers.

\section{Algorithm} \label{section_algorithm}
The key idea of the presented algorithm is to describe the parametric solution map~\eqref{eq:gen_sol_map} of the $\gamma$-parametrized Kolmogorov PDE~\eqref{kol_pde_intro} as the regression function of an appropriately chosen supervised statistical learning problem. One can then use simulated training data in order to learn $\bar{u}$ by means of deep learning. Inspired by the Feynman-Kac formula \eqref{feynman-Kac}, we construct a new statistical learning problem via a uniformly distributed predictor variable
and a statistically dependent target variable:
\begin{equation*} 
\Lambda := (\Gamma, X, \mathcal{T}) \in D   \times [v,w]^d \times [0,T] \quad \text{(predictor)} \quad \text{and} \quad Y:=\varphi_\Gamma  (S_{\Lambda}) \in  \mathbb{R} \quad \text{(target)}.
\end{equation*}
The random vector $S_{\Lambda}$ is defined as the value of the solution process $(S_{\Gamma, X,t})_{t \geq 0}$ of the $\Gamma$-parametrized stochastic differential equation \begin{align}
\label{equation_sde_general}
dS_{\Gamma,X,t} = \mu_\Gamma   (S_{\Gamma,X,t}) dt +  \sigma_\Gamma  (S_{\Gamma,X,t}) dB_t, \quad  S_{\Gamma,X,0} = X,
\end{align}
at the (random) stopping time $t = \mathcal{T}$.
For suitable regularity assumptions, the Feynman-Kac formula~\eqref{feynman-Kac} then ensures that
\begin{equation*} 
    \mathbbm{E}[ \,Y \ \vert \ \Lambda = (\gamma,x,t)] =\mathbbm{E}[\varphi_\Gamma  (S_{\Lambda} ) \ \vert \ \Lambda = (\gamma,x,t)] = \mathbb{E}[\varphi_\gamma  (S_{\gamma,x,t})] = u_\gamma (x, t) = \bar{u}(\gamma, x, t).
\end{equation*}
This shows that the minimizer of the corresponding statistical learning problem with quadratic loss function is indeed the parametric Kolmogorov PDE solution map, see Theorem~\ref{thm:app_learn_prob} in the appendix for the precise assumptions and a detailed proof.
\begin{theorem}[Learning Problem] \label{thm:learn_prob}
It holds that the parametric solution map $\bar{u}$ is the unique minimizer of the statistical learning problem
\begin{equation} \label{eq:sol_learn}
  \operatorname{min}_{f} \mathbbm{E}\big[ \big(f(\Lambda) -Y\big)^2  \big].
\end{equation}
\end{theorem}
Restricting to a hypothesis space of suitable neural networks $\mathcal{H}$ and minimizing the empirical mean squared error (MSE) loss corresponding to~\eqref{eq:sol_learn}, we arrive at the feasible supervised ERM problem
\begin{equation}
\label{eq:emp_learn}
\textstyle
    \operatorname{min}_{\Phi\in\mathcal{H}} \tfrac{1}{s}\sum_{i=1}^s(\Phi(\lambda_i) - y_i)^2
\end{equation}
where $((\lambda_i, y_i))_{i=1}^s$ are realizations of i.i.d.\@ samples drawn from the distribution of $(\Lambda,Y)$. Typically, this problem is then solved by a variant of stochastic gradient descent~\cite{Ruder2016}.
The algorithm is graphically illustrated in Figure~\ref{fig:alg}.
 \begin{figure}[!tb]
    \centering
    \begin{minipage}{.48\textwidth}
        \centering
        \includegraphics[width=0.71\linewidth]{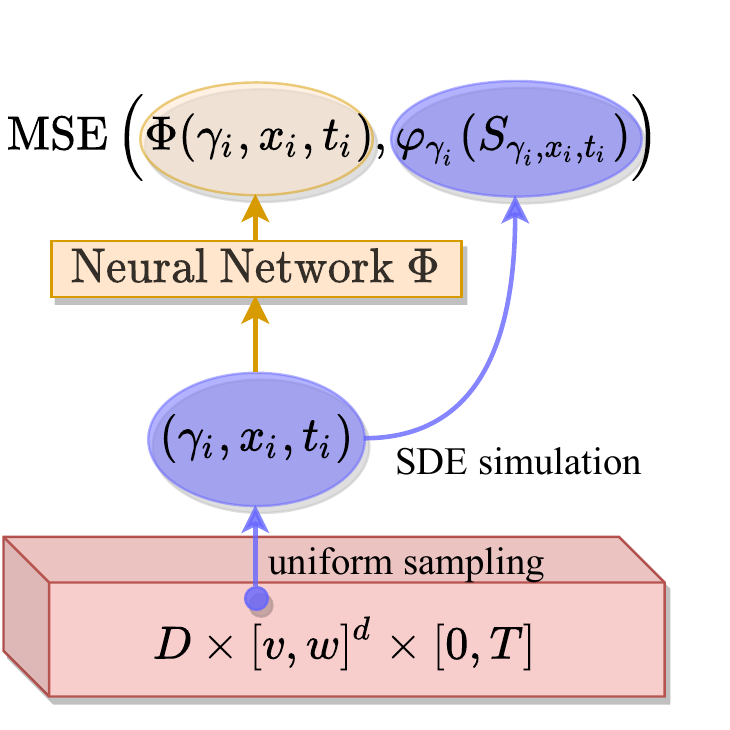}
        \caption{Illustration of the proposed supervised learning problem with predictor variable $\Lambda$ and target variable $\varphi_\Gamma  (S_{\Gamma,X,\mathcal{T}} )$.}
        \label{fig:alg}
    \end{minipage}
    \hspace{2em}
    \begin{minipage}{0.42\textwidth}
        \centering
        \includegraphics[width=\linewidth]{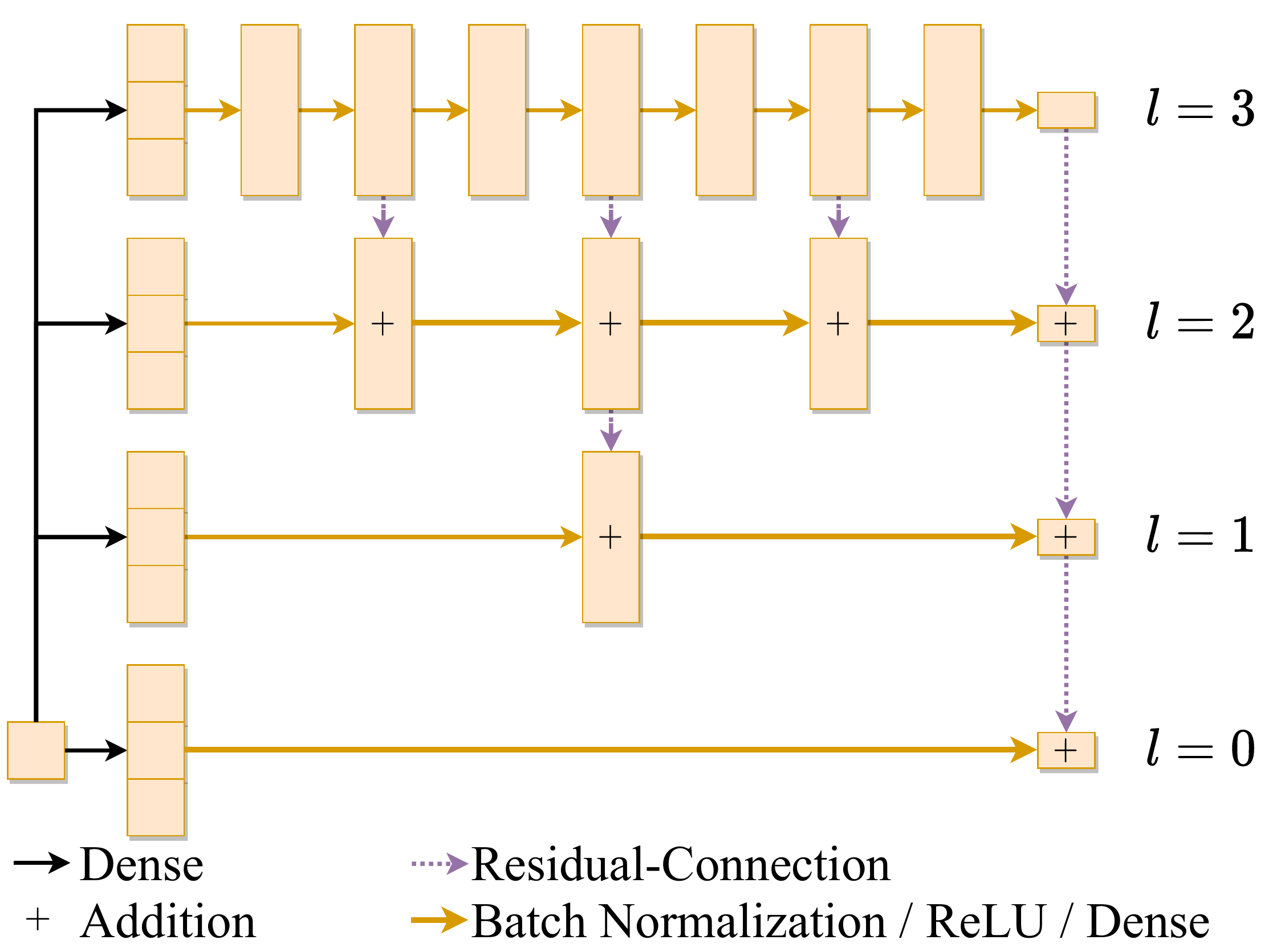}
        \caption{Illustration of the Multilevel architecture for $L=4$, $q=3$.}
        \label{fig:net}
    \end{minipage}
\end{figure}

It is trivial to simulate i.i.d.\@ samples of the predictor variable $\Lambda$, due to its uniform distribution. On the other hand i.i.d.\@ samples of the target variable $Y= \varphi_\Gamma  (S_{\Lambda} )$ can be obtained via standard numerical SDE solution techniques without curse of dimensionality~\cite{KloedenPlaten1992}. An example for such a technique is given by the \emph{Euler-Maruyama approximation} with $M\in\N$ equidistant steps $(S^{M,m}_{\Lambda})_{m=0}^M$ which is defined by the following scheme:
\begin{equation}
\label{eq:EM}
    S^{M,0}_{\Lambda} = X \quad \text{and} \quad
     S^{M,m+1}_{\Lambda} = S^{M,m}_{\Lambda} + \mu_\Gamma   (S^{M, m}_{\Lambda}) \tfrac{\mathcal{T}}{M} + \sigma_\Gamma  (S^{M, m}_{\Lambda})\big(
        B_{\frac{(m+1)\mathcal{T}}{M}} - B_{\frac{m\mathcal{T}}{M}}\big).
\end{equation}

The following theorem shows that solving the learning problem with data simulated by the Euler-Maruyama scheme does indeed result in the expected approximation of the parametric PDE solution map $\bar{u}$, see Theorem~\ref{thm:app_learn_approx} in the appendix for a proof.
\begin{theorem}[Approximated Learning Problem]
The unique minimizer $\bar{u}^{M}$ of the approximated statistical learning problem
\begin{equation*}
     \operatorname{min}_{f} \mathbbm{E}\big[ \big(f(\Lambda) -Y^M\big)^2  \big]
\end{equation*}
where $Y^M := \varphi_\Gamma  (S^{M,M}_{\Lambda} )$ is simulated using the Euler-Maruyama scheme~\eqref{eq:EM} with $M \sim 1/\eps^2$ equidistant steps satisfies that
\begin{equation*}
    \| \bar{u}^{M} - \bar{u}\|_{\mathcal{L}^\infty(D\times [v,w]^d \times [0,T])} \le \eps.
\end{equation*}
\end{theorem}
In other words, the approximation of the SDE solution $S^{M,M}_{\Lambda}\approx S_{\Lambda}$ carries over to the approximation of the corresponding minimizer $\bar{u}^{M}\approx \bar{u}$.
Therefore there exists no constraint of having to solve the SDE in~\eqref{equation_sde_general} analytically. 
The ability to easily simulate artificial training data opens up the highly desirable capability to supply the learning algorithm with a potentially infinite stream of i.i.d.\@ data samples. Instead of having to use a train/val/test split on a given finite data set, one can thus constantly simulate new data points on demand during training. Since the number of samples then grows at will parallel to the training process, the first epoch never finishes and every new gradient computation can be done on the basis of previously unseen data. 

\subsection{Example: Affine-Linear Coefficient Functions}
\label{section_affine_linear}

In Section~\ref{num_results} we present numerical experiments based on our algorithm for the important special case where $\sigma_\gamma  $ and $\mu_\gamma  $ are affine-linear functions. Thus from now on let us assume that\footnote{We denote by $[a_1\vert \dots \vert a_d] \in \R^{d\times d}$ the horizontal concatenation of the vectors $a_1,\dots, a_d \in\R^d$.}
\begin{align*}
&\sigma_\gamma  : \mathbb{R}^d \rightarrow \mathbb{R}^{d \times d}, \quad \sigma_\gamma  (x) = [\gamma_{\sigma,1}x \vert \dots \vert \gamma_{\sigma,d}x] +\gamma_{\sigma, d+1}, \\
&\mu_\gamma  : \mathbb{R}^d \rightarrow \mathbb{R}^{d}, \quad \mu_\gamma  (x) = \gamma_{\mu,1}x + \gamma_{\mu,2},
\end{align*}
are affine-linear functions, which are
determined by parameter tuples of matrices and vectors
\begin{equation*}
\gamma_{\sigma} \in D_{\sigma} \subseteq (\mathbb{R}^{d \times d})^{d+1}, \quad \gamma_{\mu} \in D_{\mu} \subseteq \mathbb{R}^{d \times d} \times \mathbb{R}^d.
\end{equation*}
The parameter sets $D_{\sigma}$ and $D_{\mu}$ are chosen to be compact. Together with a suitable compact parameter set $D_{\varphi} \subseteq \mathbb{R}^{k}$ for the initial function $\varphi_\gamma  $, we obtain
\begin{equation*}
    \gamma := (\gamma_{\sigma}, \gamma_{\mu}, \gamma_{\varphi}) \in D_{\sigma} \times D_{\mu} \times D_{\varphi} := D  \subseteq (\mathbb{R}^{d \times d})^{d+1}  \times (\mathbb{R}^{d \times d} \times \mathbb{R}^d) \times \mathbb{R}^{k}. 
\end{equation*}
This leads to an input dimension of our neural network $\Phi$ of
\begin{equation*}
\begin{aligned}
\text{dim}_{\text{in}}(\Phi) &= \text{dim}(D   \times [v,w]^d \times [0,T]) = d^3 + 2d^2 + 2d + 1 + k.
\end{aligned}
\end{equation*}
Kolmogorov PDEs with affine-linear coefficient functions regularly appear in applications; the heat equation from physics and the classical and generalized Black-Scholes equations from computational finance are important examples of Kolmogorov PDEs with affine-linear coefficient maps~\cite{EKSTROM2010498, pironneau09}. Note that while affine-linear coefficient functions are important in practice, computationally fast to evaluate, and easy to parametrize, the presented method is not restricted to the case of affine-linear coefficients and can as well be used in a substantially more general setting.
\section{Numerical Results} \label{num_results}
We implemented the framework described in Section \ref{section_algorithm} in PyTorch~\cite{NEURIPS2019_9015} and computed our results on a Nvidia DGX-1 using Tune~\cite{liaw2018tune} for experiment execution and hyperparameter optimization. In this section, we describe our setting and present four encouraging demonstrations of the performance of our algorithm.\footnote{For the implementation details we refer the reader to Section~\ref{sec:impl} in the appendix and the repository associated with this work on~\url{https://github.com/juliusberner/deep_kolmogorov}.}

For the neural network $\Phi$ we propose a \emph{Multilevel architecture} which is inspired by multilevel techniques such as Multilevel Monte Carlo methods~\cite{giles_2015}, network architectures in~\cite{he2016deep, yu2018deep}, and the architecture for the squaring function used in the proofs of our theoretical results in Section~\ref{sec:theory}, see also~\cite[Figure 2c]{yarotsky2017error}.
One can view the output of the network $\Phi$ as a sum $\sum_{l=0}^{L-1} \Phi_l$ of sub-networks $\Phi_l$ with $2^l$ layer each. We think of the shallow network $\Phi_0$ as computing a coarse approximation of $\bar{u}$ and of the deep networks $\Phi_l$ (with $l\ge1$) as approximately learning the residuals
$\bar{u} - \sum_{i=0}^{l-1} {\Phi_i}$. To facilitate optimization we normalize our inputs\footnote{We know the underlying (uniform) distribution and therefore can normalize each input in an exact manner.} and to enhance expressivity we add an initial layer which increases the width by a given factor $q$. The Multilevel architecture is depicted in Figure~\ref{fig:net} and described in detail in Definition~\ref{def:multi_lvl} in the appendix. 

Our optimized hyperparameters as well as an ablation study of our architecture and training scheme can 
be found in Sections~\ref{sec:impl} and~\ref{sec:add_num} in the appendix.
For all our experiments we were able to stick to a similar setup which depicts its robustness and general applicability. 
This is also mirrored by the small standard deviations of our considered errors across independent runs in Tables~\ref{table:bs},~\ref{table:basket},~\ref{table:heat}, and~\ref{table:heat_gaussian}. 
These tables report average runtimes (in seconds), average approximation errors, and their standard deviations w.r.t.\@ $4$ independent runs each $4000$ gradient descent steps.
As an evaluation metric we approximately computed $\mathcal{L}^1$-errors via Monte Carlo sampling, that is
\begin{equation}
 \label{eq:lp_err}
 \textstyle
   \left\|\tfrac{\Phi(\Lambda) - \bar{u}(\Lambda)}{1+|\bar{u}(\Lambda)|} \right\|_{\mathcal{L}^1} :=  \mathbbm{E}\left[\tfrac{|\Phi(\Lambda) - \bar{u}(\Lambda)|}{1+|\bar{u}(\Lambda)|} \right]  
   \approx   \tfrac{1}{n} \sum_{i=1}^n \tfrac{|\Phi(\lambda_i) - \bar{u}(\lambda_i)|}{1+|\bar{u}(\lambda_i)|} 
\end{equation}
with 
$n\in\mathbb{N}$ realizations $(\lambda_i)_{i=1}^n$ of i.i.d.\@ samples drawn from the distribution of $\Lambda$ (drawn independently of the training data in~\eqref{eq:emp_learn} and drawn independently for each evaluation step). This means that we always evaluate our model w.r.t.\@ to the parametric solution map $\bar{u}$ on unseen input data; if no closed-form solution for $\bar{u}$ is available, as in the case of the Basket option in Section~\ref{sec:basket} below, we approximate $\bar{u}(\lambda_i)$ pointwise via Monte Carlo sampling, i.e.\@
\begin{equation}
\label{eq:MC}
\textstyle
    \bar{u}(\lambda_i)=\bar{u}(\gamma_i,x_i,t_i) = \E [\varphi_{\gamma_i}(S_{\gamma_i,x_i,t_i})] \approx \tfrac{1}{m} \sum_{j=1}^m \varphi_{\gamma_i}(s_j)
\end{equation}
where $(s_j)_{j=1}^m$ are realizations of i.i.d.\@ samples drawn from the distribution of the Euler-Maruyama approximation $S^{M,M}_{\lambda_i}$ (drawn independently of the training data in~\eqref{eq:emp_learn} and drawn independently for each point and each evaluation step).
Note that~\eqref{eq:lp_err} is invariant under scaling of the hypercubes and locally corresponds to relative errors where the solution $\bar{u}$ is large and absolute errors where it is small, which in particular prevents division by zero.

\subsection{Black-Scholes Options Pricing Model} \label{section_black_scholes}
Our first example shows that neural networks are capable of learning a parametric version of the highly-celebrated Black-Scholes option pricing model~\cite{black1973pricing}.
We consider a European put option which gives its owner the right, but not the obligation, to sell a single
underlying financial asset at a specified strike price and at a given time. Formally, this corresponds to $d=1$ and
\begin{equation*}
    \sigma_\gamma  (x) = \gamma_{\sigma}x, \quad \mu_\gamma  (x) = 0, \quad    \varphi_\gamma(x) = \max\{ \gamma_\varphi-x ,0\}, \quad x\in\R,
\end{equation*}
with\footnote{Note that $\sigma_\gamma = \sigma_{\gamma,1}$ in the formal framework described in Section \ref{section_affine_linear} but here and in the following we use the natural identifications, e.g.\@ $D_\sigma\cong D_\sigma \times \{0\}.$} $\gamma_\sigma \in D_\sigma \subseteq \R$ and $\gamma_\varphi \in D_\varphi \subseteq \R$.
Effectively, this leads to an input dimension of our neural network $\Phi$ of 
$
\text{dim}_{\text{in}}(\Phi) = 4. 
$
In case of the present Black-Scholes model, the associated SDE in~\eqref{equation_sde_general} can actually be solved explicitly; it gives rise to geometric Brownian motion with uniformly distributed volatility $\Gamma_\sigma \in D_\sigma $, initial value $X\in [v,w]$, and stopping time $\mathcal{T} \in [0,T]$, i.e.
\begin{equation*} 
    S_{\Lambda}  = Xe^{-0.5\mathcal{T}\, \Gamma_\sigma^2 + \sqrt{\mathcal{T}} \, \Gamma_\sigma N}
\end{equation*}
where $N\sim \mathcal{N}(0,1)$ is normally distributed and independent of $\Lambda$.
We thus obtain an analytic expression for the parametric PDE solution,
\begin{equation*}
    \bar{u}(\gamma,x,t) = \gamma_\varphi \Psi(h_\gamma(x, t)+ \sqrt{t} \, \gamma_\sigma) -  x \Psi (h_\gamma(x,t)),
\end{equation*}
and the partial derivatives, e.g.\@
\begin{equation*}
    \tfrac{\partial \bar{u}}{\partial \gamma_\sigma}(\gamma,x,t) = x \sqrt{t} \Psi'(-h_\gamma(x,t)),
\end{equation*}
where 
\begin{equation*}
\Psi(z):= \tfrac{1}{2}\big(1 +\operatorname{erf}\big(\tfrac{z}{\sqrt{2}}\big)\big) \quad \text{and} \quad  h_\gamma(x,t) := -\tfrac{1}{\sqrt{t} \, \gamma_\sigma} \big( \ln\big(\tfrac{x}{\gamma_\varphi}\big) + \tfrac{t\gamma_\sigma^2}{2}\big), 
\end{equation*}
see~\cite[Section 13.7]{baldi2017stochastic}. This analytic expression can be used to evaluate the performance of our algorithm.
We point out that the partial derivatives of $\bar{u}$ are crucial in option pricing and each of them is associated with a distinct economic interpretation. They are often referred to as \textit{Greeks} and describe the sensitivity of the option price w.r.t.\@ different model parameters, see for instance \cite{baldi2017stochastic, seydel2006tools}. The most prominent Greeks are given by
\begin{equation*}
    \Delta = \tfrac{\partial \bar{u}}{\partial x},  \quad   \operatorname{Vega} = \tfrac{\partial \bar{u}}{\partial \gamma_\sigma}, \quad \Theta = -\tfrac{\partial \bar{u}}{\partial t}.
\end{equation*}
On the basis of the proposed algorithm, our neural network $\Phi$ is capable of learning the parametric solution map $\bar{u}$ of the above problem in $24000$ gradient updates up to an average $\mathcal{L}^1$-error of $0.0011$, see Table~\ref{table:bs} and Figure~\ref{fig:error_bs1}. As expected, the partial derivatives of our network $\Phi$ (computed via automatic differentiation) approximate the partial derivatives of $\bar{u}$ as can be seen in Figure~\ref{fig:error_greek1}. Further evidence can be found in Figures~\ref{fig:bs2},~\ref{fig:greek2},~\ref{fig:error_bs2}, and~\ref{fig:error_greek2} in the appendix.
Even though the parametric PDE problem can be solved explicitly in this special case, 
we use this relatively simple example for the purpose of illustrating our algorithm in an intuitive setting. As we will see below, the proposed algorithm is by no means restricted to such basic examples and can be applied successfully to much more complex and high-dimensional problems as well.
\begin{figure}[!tb]
    \centering
    \begin{minipage}[b]{0.45\textwidth}
        \centering
        \includegraphics[width=\linewidth]{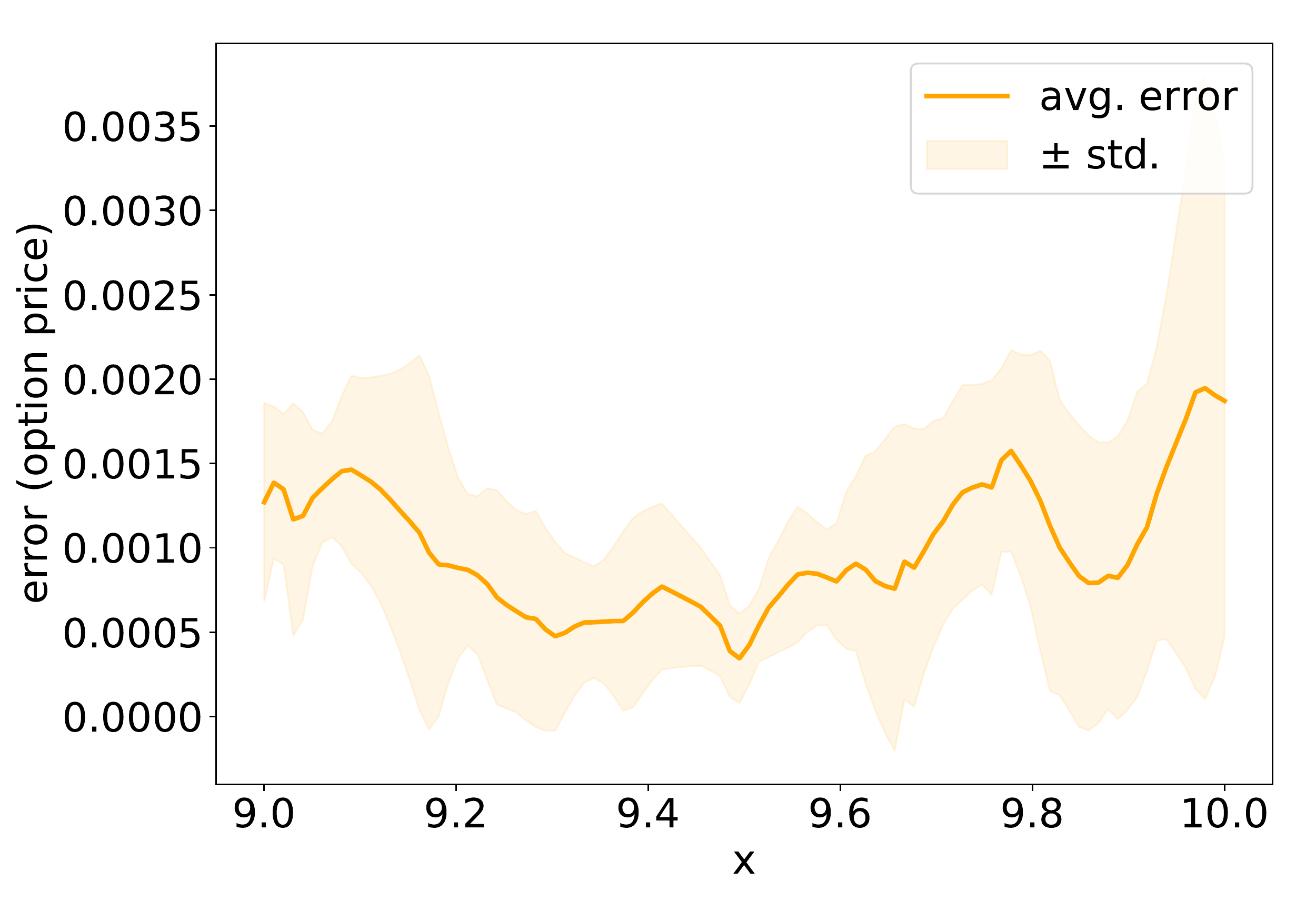}
        \caption{Shows the average prediction error $\tfrac{\vphantom{\frac{\partial \Phi}{\partial \gamma_\sigma}}|\Phi(\gamma, {\cdot}, t) -\bar{u}(\gamma, {\cdot}, t)|}{\vphantom{\frac{\partial \Phi}{\partial \gamma_\sigma}}1+|\bar{u}(\gamma, {\cdot}, t)|}$ and its standard deviation at $t=0.5$, $\gamma_\sigma=0.35$, and $\gamma_\varphi=11$.}
        \label{fig:error_bs1}
    \end{minipage}
    \hspace{2em}
    \begin{minipage}[b]{0.45\textwidth}
        \centering
        \includegraphics[width=\linewidth]{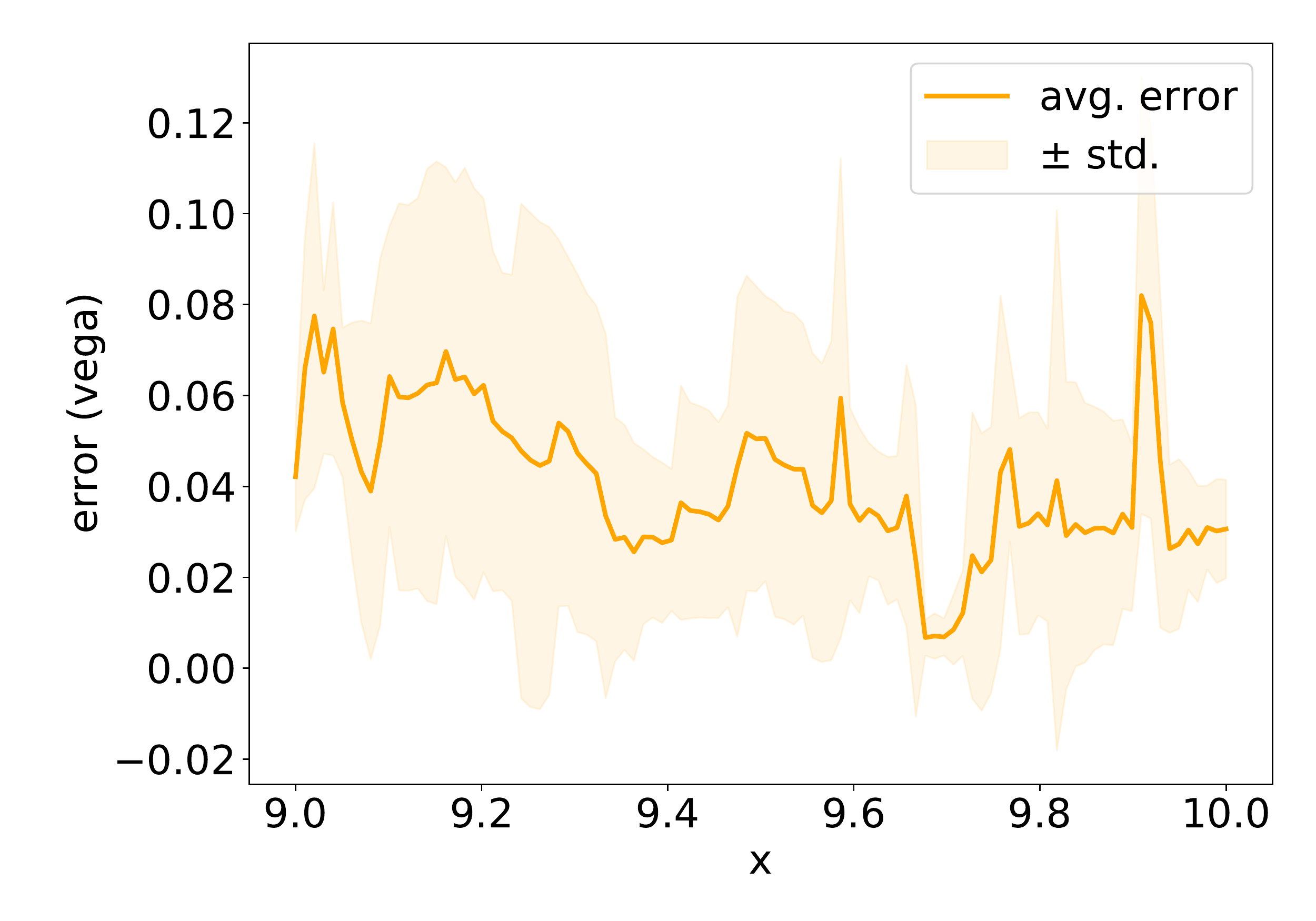}
        \caption{Shows the average error of the Vega $\tfrac{|\frac{\partial \Phi}{\partial \gamma_\sigma}
        (\gamma, {\cdot}, t) -\frac{\partial \bar{u}}{\partial \gamma_\sigma}(\gamma, {\cdot}, t)|}{1+|\frac{\partial \bar{u}}{\partial \gamma_\sigma}(\gamma, {\cdot}, t)|}$ and its standard deviation at $t=0.5$, $\gamma_\sigma=0.35$, and $\gamma_\varphi=11$.}
        \label{fig:error_greek1}
    \end{minipage}
\end{figure}
\subsection{Basket Put Option}
\label{sec:basket}
In the following we show that we can obtain comparable results to Section \ref{section_black_scholes} in the case of a considerably more complicated Basket put option pricing problem, where analytical solutions of the PDE and the SDE are lacking. In such cases, our algorithm allows practitioners to nevertheless gain valuable insights into the behaviour of the PDE solution manifold as input parameters vary. By means of our trained model $\Phi$ one can easily compute sensitivity values $\tfrac{\partial\Phi}{\partial\gamma}\approx \tfrac{\partial\bar{u}}{\partial\gamma}$, $\tfrac{\partial\Phi}{\partial t}\approx \tfrac{\partial\bar{u}}{\partial t}$, and $\tfrac{\partial\Phi}{\partial x}\approx \tfrac{\partial\bar{u}}{\partial x}$ via automatic differentiation or fit the parameter $\gamma$ to a real-world data-set $((x_i,t_i), u_\gamma(x_i,t_i))_{i=1}^m$ with unknown $\gamma$ by minimizing $\min_{\gamma \in D} \sum_{i=1}^m \big(\Phi(\gamma, x_i, t_i)-u_\gamma(x_i,t_i)\big)^2$ via stochastic gradient descent. Moreover, one can obtain estimates for probabilistic quantities related to uncertainty such as
\begin{equation*}
\textstyle
   \mathbbm{V}[u_\Xi(x,t)]\approx\mathbbm{V}[\Phi(\Xi,x,t)]\approx \tfrac{1}{m-1}\sum_{i=1}^m \big(\Phi(\xi_i, x, t)-\tfrac{1}{n}\sum_{j=1}^m\Phi(\xi_j, x, t)\big)^2 
\end{equation*}
where $(\xi_i)_{i=1}^m$ are realizations of i.i.d.\@ samples drawn from the distribution of a random variable $\Xi$ of interest. None of these types of insights were accessible before the presented deep learning method.

We proceed by demonstrating the performance of the proposed algorithm for a general multidimensional affine-linear setting as described in Section~\ref{section_affine_linear}. To this end, let $d=3$ and define the initial condition via
\begin{equation*}
\textstyle
   \varphi_\gamma(x) = \max\big\{ \gamma_\varphi-\tfrac{1}{3}\sum_{i=1}^3 x_i ,0\big\}, \quad x \in \R^3,
\end{equation*}
with $\gamma_{\varphi} \in D_{\varphi} \subseteq \mathbb{R}$.
This corresponds to the situation of a Basket put option in a very general multidimensional Black-Scholes model with $3$ potentially highly correlated assets. Note that within the above setup, the input dimension of our neural network $\Phi$ is given by
\begin{equation*}
\text{dim}_{\text{in}}(\Phi) = d^3 + 2d^2 + 2d + 1 + 1 = 53. 
\end{equation*}
To generate samples of our target variable $\varphi_{\Gamma}(S_{\Lambda})$, we simulate solutions of the SDE in~\eqref{equation_sde_general} using the Euler-Maruyama scheme~\eqref{eq:EM} with $M=25$ equidistant steps. Moreover, we use a Monte Carlo approximation with $m=2^{20}$ samples to compute the pointwise evaluation of the reference solution $\bar{u}(\lambda_i)$ according to~\eqref{eq:MC} as needed for the error estimation in~\eqref{eq:lp_err}. Despite the considerably higher dimension of this problem compared with the previous problem from Section~\ref{section_black_scholes}, our deep learning approach shows almost the same approximation behavior, see Table~\ref{table:basket}. This remarkably weak dependence on the dimension of the input data is further supported by the next examples from physical modelling, where we shall increase the dimensionality of the studied problems even further.

\begin{table}[!tb]
    \centering
    \begin{minipage}{.45\textwidth}
        \centering
        \caption{\vphantom{p}Results for the Black-Scholes model}
\begin{tabular}{llll}
\toprule
  step & avg. time (s)  &    avg. $\mathcal{L}^1$-error \\
\midrule
     0 &          0 $\pm$ 0 &  0.6812 $\pm$ 0.0704 \\
  4k &        471 $\pm$ 3 &  0.0088 $\pm$ 0.0056  \\
  8k &        943 $\pm$ 6 &  0.0062 $\pm$ 0.0025 \\
 12k &       1413 $\pm$ 9 &   0.0030 $\pm$ 0.0004 \\
 16k &      1885 $\pm$ 11 &  0.0017 $\pm$ 0.0001 \\
 20k &      2356 $\pm$ 14 &  0.0013 $\pm$ 0.0002 \\
 24k &      2827 $\pm$ 17 &  0.0011 $\pm$ 0.0001 \\
\bottomrule
&     &  \\
\end{tabular}
    \label{table:bs}
    \end{minipage}
    \hspace{2em}
    \begin{minipage}{0.45\textwidth}
        \centering
        \caption{Results for the Basket option}
\begin{tabular}{llll}
\toprule
  step & avg. time (s) &    avg. $\mathcal{L}^1$-error   \\
\midrule
     0 &          0 $\pm$ 0 &  0.7912 $\pm$ 0.0276  \\
  4k &        811 $\pm$ 7 &  0.0131 $\pm$ 0.0019  \\
  8k &       1614 $\pm$ 4 &  0.0087 $\pm$ 0.0013  \\
 12k &      2434 $\pm$ 28 &  0.0062 $\pm$ 0.0009  \\
 16k &      3236 $\pm$ 27 &  0.0058 $\pm$ 0.0011  \\
 20k &     4162 $\pm$ 154 &  0.0046 $\pm$ 0.0007  \\
 24k &     5077 $\pm$ 291 &  0.0042 $\pm$ 0.0002  \\
 28k &     6024 $\pm$ 463 &  0.0039 $\pm$ 0.0001  \\
\bottomrule
\end{tabular}
    \label{table:basket}
    \end{minipage}
\end{table}
\subsection{Heat Equation with Varying Diffusion Coefficients} 
\label{sec:heat}
In this Section, we present two examples of high-dimensional heat equations in $d =10$ and $d=150$ dimensions with paraboloid and Gaussian initial conditions
\begin{equation*}
    \varphi_\gamma(x) = \|x\|^2 \quad \text{(paraboloid)} \quad \text{and} \quad \varphi_\gamma = e^{-\|x\|^2} \quad  \text{(Gaussian)}.
\end{equation*}
This formally corresponds to
\begin{equation*}
    \sigma_\gamma  (x) = \gamma_\sigma \quad  \text{and} \quad \mu_\gamma  (x) = 0
\end{equation*}
where we use a matrix
$\gamma_\sigma \in D_\sigma \subseteq \R^{10\times 10}$ for the paraboloid case and a scalar $\gamma_\sigma \in D_\sigma \subseteq \R$ in the Gaussian case, leading to 
input dimensions of our models $\Phi$ of
\begin{equation*}
\text{dim}_{\text{in}}(\Phi) = d^2 + d + 1 = 111 \quad \text{(paraboloid)} \quad \text{and} \quad  \text{dim}_{\text{in}}(\Phi) = d + 1 + 1 = 152 \quad \text{(Gaussian)}.
\end{equation*}

Notice that here the solution of the corresponding SDE can be directly sampled via a Brownian motion with uniformly distributed scaling $\Gamma_\sigma$, initial position $X$, and stopping time $\mathcal{T}$, i.e.
\begin{equation*} 
  S_{\Lambda}  =   X +  \sqrt{\mathcal{T}} \, \Gamma_\sigma N
\end{equation*}
where $N\sim \mathcal{N}(0,I_d)$ is normally distributed and independent of $\Lambda$, see~\cite[Section 3.2]{beck2018solving}.
For evaluation purposes, these examples were purposefully constructed to have analytic expressions for the parametric solution maps $\bar{u}$, which are given by
\begin{equation*} 
    \bar{u}(\gamma_{\sigma}  ,x,t) = \|x\|^2 + t \operatorname{Trace}(\gamma_\sigma \gamma_\sigma^{*}) \quad \text{(paraboloid)}, \quad \bar{u}(\gamma_{\sigma}  ,x,t) = \tfrac{e^{-\tfrac{\|x\|^2}{1 + 2t\gamma_{\sigma}^2}}}{(1 + 2t\gamma_{\sigma}^2)^{d/2}} \quad \text{(Gaussian)}.
\end{equation*}
However, in almost all other practical cases an analytic solution for $\bar{u}$ is impossible to obtain and numerical methods are the only path forward.

\begin{table}[!tb]
    \centering
    \begin{minipage}{.45\textwidth}
        \centering
        \caption{Results for the heat equation with paraboloid initial condition}
        \begin{tabular}{llll}
\toprule
  step & avg. time (s) &    avg. $\mathcal{L}^1$-error \\
\midrule
     0 &          0 $\pm$ 0 &  0.9609 $\pm$ 0.0052\\
  4k &      1904 $\pm$ 19 &  0.0150 $\pm$ 0.0008\\
  8k &      3808 $\pm$ 37 &  0.0120 $\pm$ 0.0007\\
 12k &      5712 $\pm$ 57 &  0.0093 $\pm$ 0.0006  \\
 16k &      7616 $\pm$ 76 &  0.0068 $\pm$ 0.0001  \\
 20k &      9520 $\pm$ 95 &  0.0062 $\pm$ 0.0003  \\
 24k &    11424 $\pm$ 114 &  0.0057 $\pm$ 0.0001  \\
 28k &    13328 $\pm$ 132 &  0.0056 $\pm$ 0.0000  \\
\bottomrule
\end{tabular}
    \label{table:heat}
    \end{minipage}
    \hspace{2em}
    \begin{minipage}{0.45\textwidth}
        \centering
        \caption{Results for the heat equation with Gaussian initial condition\vphantom{p}}
         \begin{tabular}{llll}
\toprule
  step & avg. time (s) &    avg. $\mathcal{L}^1$-error   \\
\midrule
     0 &          0 $\pm$ 0 &  0.2035 $\pm$ 0.0714 \\
  4k &      2070 $\pm$ 40 &  0.0123 $\pm$ 0.0047 \\
  8k &      4131 $\pm$ 82 &  0.0050 $\pm$ 0.0018 \\
 12k &     6192 $\pm$ 124 &  0.0051 $\pm$ 0.0022 \\
 16k &     8258 $\pm$ 165 &  0.0033 $\pm$ 0.0015 \\
 20k &    10323 $\pm$ 206 &  0.0025 $\pm$ 0.0011 \\
 24k &    12388 $\pm$ 247 &  0.0024 $\pm$ 0.0008 \\
 28k &    14454 $\pm$ 290 &  0.0019 $\pm$ 0.0002 \\
\bottomrule
\end{tabular}
    \label{table:heat_gaussian}
    \end{minipage}
\end{table}

The above dimensionality settings represent regimes which are completely out of scope for all preexisting numerical schemes. Nevertheless, Tables~\ref{table:heat} and~\ref{table:heat_gaussian} confirm that our proposed deep learning method once again efficiently converges to the desired parametric solution map $\bar{u}$. 
Our results empirically demonstrate that, contrary to conventional numerical solvers, our deep learning based method does not suffer from the curse of dimensionality, see also Figure~\ref{fig:heat_dims} in the appendix. 
We will rigorously prove this fact in the next section.

\section{Theoretical Guarantees}

As a first example, we stick to the heat equation with paraboloid initial condition from above and show that neural networks are capable of simultaneously approximating the parametric solution map $\bar{u}$ and its gradient with the number of network parameters scaling only polynomially in the dimension $d$, see Theorem~\ref{thm:app_approx_para} in the appendix for a proof.
Such an approximation guarantee without curse of dimensionality ensures that sensitivity analysis is possible even in very high dimensions.
\begin{theorem}[Sobolev Approximation]
There exists a neural network $\Phi$ with ReLU activation function and $\mathcal{O}(d^4\log(d/\eps))$ parameters satisfying that
\begin{equation*}
\|\Phi - \bar{u}\|_{\mathcal{L}^{\infty}(D   \times [v,w]^d \times [0,T])}\le \eps \quad  \text{and} \quad \|\nabla\Phi - \nabla\bar{u}\|_{\mathcal{L}^{\infty}(D   \times [v,w]^d \times [0,T])}\le \eps.
\end{equation*}
\end{theorem}
Let us now consider the heat equation with varying diffusivity and Gaussian initial condition.
In fact, our framework allows us to rigorously prove sample complexity estimates for this problem which represents an almost unique scenario for deep learning based methods. This is rendered possible by the structure of the underlying parametric Kolmogorov PDE and its associated SDE which allows us to describe the distribution of the predictor and target variable, simulate i.i.d.\@ samples, and infer regularity properties on the regression function. We briefly sketch the theorem in the following; the precise formulation and the proof is given in Theorem~\ref{thm:app_gen} in the appendix. 
\begin{theorem}[Generalization]
\label{thm:gen}
Using $s\sim(d/\eps)^2\operatorname{polylog}(d/\eps)$ many samples, every empirical risk minimizer $\hat{\Phi}$ of~\eqref{eq:emp_learn}
in a suitable hypothesis space $\mathcal{H}$ of neural networks with ReLU activation function, $\mathcal{O}(\operatorname{polylog}(d/\eps))$ layers, $\mathcal{O}(d)$ neurons per layer, and parameters bounded by $\mathcal{O}(1)$ satisfies with high probability that
\begin{equation*}
    \tfrac{1}{V}\|\hat{\Phi} - \bar{u}\|^2_{\mathcal{L}^2(D   \times [v,w]^d \times [0,T])} \le \eps
\end{equation*}
where $V:=\operatorname{vol}(D   \times [v,w]^d \times [0,T])$.
\end{theorem}
Note that it holds that 
\begin{equation*}
    \tfrac{1}{V}\|\cdot\|^2_{\mathcal{L}^2 (D\times [v,w]^d \times [0,T])}  = \|\cdot\|^2_{\mathcal{L}^2 (\P_\Lambda)}
\end{equation*}
where $\P_\Lambda$ is the uniform probability measure on $D\times [v,w]^d \times [0,T]$. Thus the estimate in Theorem~\ref{thm:gen} can be viewed as an estimate in the space $\mathcal{L}^2(\P_\Lambda)$ and we want to emphasize that our setting easily allows us to choose arbitrary probability measures $\P$ on $D\times [v,w]^d \times [0,T]$ and prove analogous results w.r.t.\@ the $\mathcal{L}^2(\P)$-norm.

\section{Conclusion}

The method introduced in this paper is the first deep learning algorithm for the numerical solution of parametric Kolmogorov PDEs and one of few existing algorithms whose use is computationally tractable in high-dimensional settings. The parametric nature of our approach readily allows for sensitivity analysis, model calibration, and uncertainty quantification, all which is of high interest in a variety of applications. Successful numerical experiments in both low- and high-dimensional settings empirically confirm the functionality of the proposed algorithm. In addition, we are able to provide theoretical guarantees for the applicability of our method in high-dimensions. 

Besides solving an important problem in scientific computing, our work introduces a class of learning problems that allows for the rigorous investigation of expressivity and sample complexity, along with stable and interpretable algorithms. Such strong results become possible by leveraging the mathematical structure of the learning problem associated with the parametric PDE. We anticipate that the formulation and study of such structured problems will constitute an important future direction of research in the scientific machine learning community as it can enable reliable and interpretable algorithms for the solution of previously intractable problems: in our case parametric families of Kolmogorov PDEs. This contributes substantially to areas like physical modelling of diffusion processes and computational finance, which all rely on the use of such PDEs.

\newpage

\section*{Broader Impact}

The deep-learning technique presented in this work is the first computationally scalable method for the numerical solution of high-dimensional parametric Kolmogorov PDEs. It is also the first method which allows for a straightforward sensitivity analysis of the associated high-dimensional PDE solution manifold with respect to input parameters. In addition, it newly allows for high-dimensional data-driven model calibration and uncertainty quantification. While it is a difficult task to precisely estimate the cascading effects of technological innovations on wider society, it is reasonable to assume that the ubiquity of Kolmogorov equations in science and engineering will lead to a positive impact of our new findings on a multitude of technical areas of social importance. 

As an example, Kolmogorov PDEs are heavily used in physics for the modelling of heat flow and diffusion processes \cite{pascucci2005kolmogorov, widder1976heat}. Simultaneously, Fokker-Planck equations, which take the form of Kolmogorov equations in particular special cases, are used in the geophysical and atmospheric sciences as modelling tools for climate change projections \cite{hasselmann1976stochastic, thuburn2005climate}. 
Our described algorithm has clear promise to make previously intractable high-dimensional physical models computationally accessible to scientists. Additionally, our method allows for an easy investigation of changes in complex model forecasts as input parameters are varied during sensitivity analysis. Such advancements have the potential to accelerate scientific research and can directly lead to better predictive models in applied physics and engineering. Reliable and efficient predictive models in turn are essential to rationally inform public policy. 

A conceivable risk posed by our work might come in the form of the uncritical use of our algorithm in applications related to financial engineering. The Black-Scholes equation and associated models have been notoriously misused in the last decades by semi-technical users working in financial sectors around the world \cite{jarrow2011risk, wilmott2000use}. The naive usage of technical tools in computational finance has thus likely been a contributing factor to periods of economic instability in recent history. Our technique can now add a powerful solver for high-dimensional parametric PDE problems to the tool kits of individual end-users in finance with various degrees of scientific expertise. Inexperienced users without appropriate quantitative background might be prone to erroneously taking the complexity of a high-dimensional financial model as an indicator for its accuracy. Therefore, one must take great care to systematically inform users without suitable experience in such a scenario that merely increasing the dimension of an inadequate financial model might not necessarily make its results more accurate.

In total, we are confident that the net impact of our work on the scientific community as well as broader society is positive. The probability of uncritical use of our technique and other algorithms in financial engineering can likely be substantially mitigated by targeted educational interventions and we would encourage practical research in this direction. At the same time, we note that our technical contribution is a general-purpose tool which has the potential to stimulate the acceleration of scientific progress in a wide variety of disciplines.
\begin{ack}
The research of Julius Berner was supported by the Austrian Science Fund (FWF) under
grant I3403-N32. The research of Markus Dablander was supported by the UK EPSRC Centre For Doctoral Training in Industrially Focused Mathematical Modelling (EP/L015803/1).
\end{ack}
\bibliography{mainbib}

\newpage
\appendix 
\section{Appendix}
\subsection{Theoretical Results} \label{sec:theory}
First we state our assumptions on the coefficient maps and initial conditions.
\begin{ass}[Coefficient Maps \& Initial Conditions]
\label{ass:coeff}
Let $D$ be a compact set in Euclidean space and for every $\gamma\in D$ let $\varphi_\gamma\in \mathcal{C}(\R^d,\R)$, $\sigma_\gamma \in \mathcal{C}(\R^d, \R^{d\times d})$, and $\mu_\gamma \in \mathcal{C}(\R^d,\R^d)$. Assume that for every $x\in\R^d$ the mappings
\begin{equation*}
    \gamma \mapsto \varphi_\gamma(x), \quad \gamma \mapsto \sigma_\gamma(x), \quad \text{and} \quad \gamma \mapsto \mu_\gamma(x)
\end{equation*}
are continuous and that there exists $c\in(0,\infty)$ such that for every $\gamma\in D$, $x,y\in\mathbb{R}^d$ it holds that\footnote{For a finite index set $I$ and $a,b\in\R^I$ we define $\|a\|= \sqrt{\sum_{i\in I} |a_i|^2}$ and $\langle a, b\rangle = \sum_{i\in I} a_i b_i$.}  \begin{enumerate}[label=(\roman*)]
    \item\label{it:ass_1}  $|\varphi_\gamma(x)-\varphi_\gamma(y)| \le c \|x-y\| (1+\|x\|^c+\|y\|^c)$,
    \item\label{it:ass_2}
    $\|\mu_\gamma(x)-\mu_\gamma(y)\| + \|\sigma_\gamma(x)-\sigma_\gamma(y)\| \le c \|x-y\| $, and
    \item\label{it:ass_3} $|\varphi_\gamma(0)| + \|\mu_\gamma(0)\| + \|\sigma_\gamma(0)\| \le c $.
\end{enumerate}
\end{ass}
Note that the continuity assumptions on $\sigma_\gamma$ and $\mu_\gamma$ and the condition in Item~\ref{it:ass_2} are fulfilled for the case of affine-linear coefficient functions as described in Section~\ref{section_affine_linear} and used in our examples. Further, the polynomial growth condition on the local Lipschitz constant in Item~\ref{it:ass_1}, the uniform bound in Item~\ref{it:ass_3}, and the continuity assumption on $\varphi_\gamma$ are also satisfied for all our considered examples. Under these assumptions we can precisely formulate the setting we are working in.
\begin{definition}[Parametric Kolmogorov PDEs]
For every $\gamma \in D$ let $u_\gamma\colon\R^d \times [0,\infty) \to 
\R$ be the unique continuous, at most polynomially growing function satisfying for every $x\in\R^d$ that $u_\gamma (x,0) = \varphi_\gamma(x)$ and satisfying that $u|_{\R^d\times(0,\infty)}$ is a viscosity solution of the Kolmogorov PDE
\begin{equation*}
\tfrac{\partial u_\gamma }{\partial t} (x,t) = \tfrac{1}{2} \trace\big(\sigma_\gamma(x)  [\sigma_\gamma(x)  ]^{*}(\nabla_x^2 u_\gamma)(x,t) \big) + \langle \mu_\gamma(x) , (\nabla_x u_\gamma)(x,t) \rangle
\end{equation*}
for $(x,t)\in \R^d\times(0,\infty)$, see~\cite[Corollary 4.17]{hairer2015loss}.
Let $(\Omega,\F,(\F_t)_{t\in [0,T]},\P)$ be a suitable filtered probability space satisfying the usual conditions, let 
\begin{equation}
\label{eq:bm}
    (B_t)_{t\ge 0}\colon [0,\infty)\times\Omega\to\R^d
\end{equation}
be a standard $d$-dimensional $(\F_t)$-Brownian motion, let $T\in(0, \infty)$, $v\in\R$, $w\in (v, \infty)$ and let 
\begin{equation*}
    \Lambda = (\Gamma, X, \mathcal{T})\colon \Omega \to D   \times [v,w]^d \times [0,T]
\end{equation*}
be a $\F_0$-measurable, uniformly distributed random variable. Let 
\begin{equation*}
    (S_{\gamma,x,t})_{t\ge0}\colon [0,\infty)\times\Omega\to\R^d, \quad (\gamma, x) \in D   \times \R^d, \quad \text{and} \quad (S_{\Gamma,X,t})_{t\ge 0}\colon [0,\infty) \times \Omega\to\R^d
\end{equation*}
be the up to indistinguishability unique $(\F_t)$-adapted stochastic processes with continuous sample paths satisfying that for every $(\gamma, x,t) \in D   \times \R^d\times [0,\infty)$ it holds $\P$-a.s.\@ that
\begin{align}  \label{eq:sde_app_simp}
S_{\gamma,x,t} = x + \int_0^t \mu_\gamma   (S_{\gamma,x,s}) ds + \int_0^t \sigma_\gamma  (S_{\gamma,x,s}) dB_s, 
\end{align}
and that for every $t \in [0,\infty)$ it holds $\P$-a.s.\@ that
\begin{equation} \label{eq:sde_app_gen}
    S_{\Gamma,X,t} = X + \int_0^t \mu_\Gamma   (S_{\Gamma,X,s}) ds + \int_0^t \sigma_\Gamma  (S_{\Gamma,X,s}) dB_s,
\end{equation} 
see, for instance,~\cite[Proof of Theorem 8.3]{gall2016}. For every $M\in\N$, $(\gamma, x,t) \in D   \times \R^d\times [0,\infty)$ let 
\begin{equation*}
     (S^{M,m}_{\gamma, x, t})_{m=0}^M\colon\{0,\dots, M\} \times \Omega \mapsto \R^d
\end{equation*}
be a stochastic process satisfying that $S^{M,0}_{\gamma, x, t} = x$
and for every $m\in\{0, \dots, M-1\}$ that
\begin{equation*}
     S^{M,m+1}_{\gamma, x, t} = S^{M,m}_{\gamma, x, t} + \mu_\gamma   (S^{M,m}_{\gamma, x, t}) \tfrac{t}{M} +  \sigma_\gamma  (S^{M,m}_{\gamma, x, t})\big(
        B_{\frac{(m+1)t}{M}} - B_{\frac{mt}{M}}\big)
\end{equation*}
and for every $M\in\N$ let 
\begin{equation*}
     (S^{M,m}_{\Gamma, X, \mathcal{T}})_{m=0}^M\colon\{0,\dots, M\} \times \Omega \mapsto \R^d
\end{equation*}
be a stochastic process satisfying that $S^{M,0}_{\Gamma, X, \mathcal{T}} = X$
and for every $m\in\{0, \dots, M-1\}$ that
\begin{equation*}
     S^{M,m+1}_{\Gamma, X, \mathcal{T}} = S^{M,m}_{\Gamma, X, \mathcal{T}} + \mu_\Gamma   (S^{M,m}_{\Gamma, X, \mathcal{T}}) \tfrac{\mathcal{T}}{M} +  \sigma_\Gamma  (S^{M,m}_{\Gamma, X, \mathcal{T}})\big(
        B_{\frac{(m+1)\mathcal{T}}{M}} - B_{\frac{m\mathcal{T}}{M}}\big).
\end{equation*}
Finally, let the random variable $Y\colon\Omega \mapsto \R$ be given by
\begin{equation*}
    Y := \varphi_\Gamma(S_\Lambda) = \varphi_\Gamma(S_{\Gamma,X,\T})
\end{equation*}
and for every $M\in\N$ let the random variable $Y^M\colon\Omega \mapsto \R$ be given by
\begin{equation*}
  Y^M :=\varphi_\Gamma(S^{M,M}_\Lambda) = \varphi_\Gamma(S^{M,M}_{\Gamma, X, \mathcal{T}}). 
\end{equation*}
\end{definition}
In order to prove Theorem~\ref{thm:learn_prob} we assume the following regularity on our SDEs in~\eqref{eq:sde_app_simp} and~\eqref{eq:sde_app_gen}.
\begin{ass}[Regularity Assumptions]
\label{ass:regularity}
Assume that there exists a jointly measurable\footnote{If not further specified, we consider measurability w.r.t.\@ the corresponding Borel sigma algebras.} function 
\begin{equation*}
    \Upsilon\colon \mathcal{C}([0,T],\R^d)\times D   \times [v,w]^d \times [0,T] \to \R
\end{equation*} 
such that it holds $\P$-a.s.\@ that 
\begin{equation*}
   \Upsilon(B,\Gamma, X, \mathcal{T})= \varphi_\Gamma  (S_\Lambda)
\end{equation*}
and for every $(\gamma, x, t) \in D   \times [v,w]^d \times [0,T]$ it holds $\P$-a.s.\@ that
\begin{equation*}
     \Upsilon(B,\gamma, x, t)= \varphi_\gamma  (S_{\gamma,x,t}),
\end{equation*}
where $B\colon \Omega \to \mathcal{C}([0,T],\R^d)$, $\omega \mapsto (t\mapsto B_t(\omega))$, denotes the mapping to the sample paths of the Brownian motion in~\eqref{eq:bm}.
\end{ass}
Note that the above assumptions are satisfied for the Black-Scholes model in Section~\ref{section_black_scholes} and the heat equations in Section~\ref{sec:heat}. In the former case we can write
\begin{equation*}
    \Upsilon (b,\gamma,x,t)  = \max\{ \gamma_\varphi-xe^{-0.5t\, \gamma_\sigma^2 + \sqrt{t} \, \gamma_\sigma b(1)},0\}
\end{equation*}
and in the latter
\begin{equation*}
     \Upsilon (b,\gamma,x,t)  =  \|x + \sqrt{t} \, \gamma_\sigma b(1)\|^2 \quad \text{(paraboloid)}, \quad \Upsilon (b,\gamma,x,t)  =  e^{-\|x + \sqrt{t} \, \gamma_\sigma b(1)\|^2} \quad \text{(Gaussian)}
\end{equation*}
where $(b,\gamma,x,t) \in \mathcal{C}([0,T],\R^d)\times D   \times [v,w]^d \times [0,T]$. Moreover, the existence of a suitable $\Upsilon$ is in general given for non-parametric Kolmogorov PDEs, see~\cite[Theorem 8.5]{gall2016} and~\cite{beck2018solving}.
First we establish that under our assumptions the minimizer of the statistical learning problem is indeed the parametric Kolmogorov PDE solution map.
\begin{appendixthm}[Learning Problem] 
\label{thm:app_learn_prob}
Let Assumptions~\ref{ass:coeff} and~\ref{ass:regularity} be satisfied. Then it holds that 
\begin{equation*}
    \bar{u} : D   \times [v,w]^d \times [0,T] \rightarrow \mathbb{R}, \quad 
    (\gamma, x, t) \mapsto \bar{u}(\gamma, x, t) := u_\gamma (x,t)
\end{equation*} is the (up to sets of Lebesgue measure zero) unique minimizer of the statistical learning problem
\begin{equation} \label{eq:stat_learn_prob} 
  \operatorname{min}_{f} \mathbbm{E}\Big[ \big(f(\Lambda) -Y\big)^2  \Big]
\end{equation}
where the minimum is taken over all measurable functions $f\colon D   \times [v,w]^d \times [0,T] \to \R$.
\end{appendixthm}
\begin{proof}
Note that one can extend standard results on the moments of SDE solution processes (see~\cite[Theorems 4.5.3 and 4.5.4]{KloedenPlaten1992} and~\cite[Chapter 5, Theorem 2.3]{friedman2012}) to prove that $S_{\Lambda}$ and thus also the target variable $Y=\varphi_\Gamma  (S_{\Lambda})$ have bounded moments.
It is well-known that under this condition the (up to sets of measure zero w.r.t.\@ the distribution of $\Lambda$) unique solution of the statistical learning problem~\eqref{eq:stat_learn_prob} is given by the regression function 
\begin{equation} \label{eq:rep_reg_proof}
  f^*(\gamma,x,t) :=  \mathbbm{E}[\, Y \ \vert \ \Lambda = (\gamma,x,t)],  \quad (\gamma,x,t) \in D\times [v,w]^d \times [0,T],
\end{equation}
that is
\begin{equation*}
    f^* = \operatorname{argmin}_{f} \mathbbm{E}\Big[ \big(f(\Lambda) -Y\big)^2  \Big],
\end{equation*}
see, for instance,~\cite{cucker2002mathematical}. Moreover, the Feynman-Kac formula establishes for every $(\gamma, x, t) \in D   \times [v,w]^d \times [0,T]$ that
\begin{equation} \label{eq:rep_sol_proof}
    \mathbb{E}[\varphi_\gamma  (S_{\gamma,x,t})]  
    = u_\gamma (x, t)=\bar{u}(\gamma, x, t),
\end{equation}
see~\cite[Corollary 4.17]{hairer2015loss}.
Finally, Assumptions~\ref{ass:regularity} and the independence of $B$ and $\Lambda$ ensure that for every Borel measurable set $A\subseteq D\times [v,w]^d \times [0,T]$ it holds that
\begin{equation*}
\begin{split}
\E \big[ \1_{\{\Lambda\in A\}} \varphi_\Gamma  (S_{\Lambda} )\big] 
&= \int_A  \int_{\mathcal{C}([0,T],\R^d)}\Upsilon(b,\gamma,x,t) \, d\P_B(b) \,d\P_{(\Gamma,X,\mathcal{T})}(\gamma,x,t) \\
&=\int_A \E \big[ \varphi_\gamma  (S_{\gamma,x,t}) \big] \,d\P_{(\Gamma,X,\mathcal{T})}(\gamma,x,t)
\end{split} 
\end{equation*}
where we denote the distributions of $\Lambda$ and $B$ by $\P_{(\Gamma,X,\mathcal{T})}$ and $\P_{B}$ (Wiener measure), respectively. Together with the fact that $\Lambda$ is uniformly distributed, this proves that for almost every $(\gamma, x, t) \in D   \times [v,w]^d \times [0,T]$ it holds that
\begin{equation*} 
     \mathbbm{E}[\, Y \vert \ \Lambda = (\gamma,x,t)] = \mathbbm{E}[\varphi_\Gamma  (S_{\Lambda} ) \ \vert \ \Lambda = (\gamma,x,t)] = \mathbb{E}[\varphi_\gamma  (S_{\gamma,x,t})],
\end{equation*}
see~\cite[Chapter 4]{pollard2002} and~\cite[Theorem 13.46]{aliprantis2006infinite}. Combined with~\eqref{eq:rep_reg_proof} and~\eqref{eq:rep_sol_proof}, this proves the claim.
\end{proof}
Next, we establish the stability of the statement in Theorem~\ref{thm:app_learn_prob} w.r.t.\@ approximate data generation via the Euler-Maruyama scheme.
\begin{appendixthm}[Approximated Learning Problem]
\label{thm:app_learn_approx}
Let Assumptions~\ref{ass:coeff} and~\ref{ass:regularity} be satisfied and for every $M\in\N$ let 
\begin{equation*}
    \bar{u}^{M}\colon D   \times [v,w]^d \times [0,T] \rightarrow \mathbb{R}
\end{equation*}
be the (up to sets of Lebesgue measure zero) unique solution of the approximated learning problem
\begin{equation*}
    \min_{f} \mathbbm{E}\Big[ \big(f(\Lambda) -Y^M\big)^2  \Big]
\end{equation*}
where the minimum is taken over all measurable functions $f\colon D   \times [v,w]^d \times [0,T] \to \R$.
Then there exists a constant $C>0$ such that for every $M\in\N$ it holds that
\begin{equation*}
        \| \bar{u}^{M} - \bar{u}\|_{\mathcal{L}^\infty(D\times [v,w]^d \times [0,T])} \le \tfrac{C}{\sqrt{M}}.
\end{equation*}
\end{appendixthm}
\begin{proof}
Extending results on the Euler-Maruyama scheme (see, e.g.,~\cite[Theorem 10.2.2]{KloedenPlaten1992}) one can prove that also in the parametric case for every $p\ge2$ there exists a constant $C > 0$ such that for every $M\in\N$, $(\gamma,x,t) \in D\times [v,w]^d \times [0,T]$ it holds that
\begin{equation} \label{eq:EM_approx}
   \E\big[\|S^{M,M}_{\gamma,x,t}\|^p\big]\le C \quad \text{and} \quad \big(\E \big[\|  S^{M,M}_{\gamma,x,t} - S_{\gamma,x,t} \|^p \big]\big)^{1/p} \le \tfrac{C}{\sqrt{M}}.
\end{equation}
Similar to the proof of Theorem~\ref{thm:app_learn_prob} one can further establish that for every $M\in\N$ and almost every $(\gamma,x,t) \in D\times [v,w]^d \times [0,T]$ it holds that
\begin{equation*}
    \bar{u}^{M}(\gamma,x,t) =\mathbbm{E}[\, Y^M \vert \ \Lambda = (\gamma,x,t)] = \mathbbm{E}[\varphi_\Gamma  (S^{M,M}_{\Lambda}) \ \vert \ \Lambda = (\gamma,x,t)] = \mathbb{E}[\varphi_\gamma  (S^{M,M}_{\gamma,x,t})]
\end{equation*}
where the existence of functions $\Upsilon^M$ with analogous properties as in Assumptions~\ref{ass:regularity} is guaranteed by the Euler-Maruyama scheme.
The local Lipschitz property of $\varphi_\gamma$ now ensures that for every $M\in\N$ and almost every $(\gamma,x,t) \in D\times [v,w]^d \times [0,T]$ it holds that
\begin{equation}
\begin{split}
| \bar{u}^{M}(\gamma,x,t) - \bar{u}(\gamma,x,t)| &= \big| \E \big[ \varphi_\gamma  (S^{M,M}_{\gamma,x,t}) \big]  - \mathbb{E}[\varphi_\gamma  (S_{\gamma,x,t})] \big| \\
&\le \ c\, \E\big[\|S^{M,M}_{\gamma,x,t} - S_{\gamma,x,t} \| \big(1+\| S^{M,M}_{\gamma,x,t}\|^c + \|S_{\gamma,x,t}\|^c \big)\big]
\end{split}
\end{equation}
which together with the Cauchy-Schwarz inequality and~\eqref{eq:EM_approx} proves the theorem.
\end{proof}
Note that this result can also be used to show that our generalization result in Theorem~\ref{thm:gen} is not compromised by using data simulated by the Euler-Maruyama scheme. 

Now we outline how to prove the simultaneous approximation of the parametric solution map and its partial derivatives by a neural networks without curse of dimensionality, i.e.\@ with the network size scaling only polynomially in the underlying spatial dimension. In mathematical terms, we prove approximation results in the Sobolev norm $\|{\cdot}\|_{W^{1,\infty}}$, see~\cite{Evans2015MeasureEdition}.
As a motivating example, we take the heat equation from Section~\ref{sec:heat} and 
from now on we only consider feed-forward neural networks with ReLU activation function (ReLU networks), see e.g.\@~\cite[Section 2]{petersen2018optimal} for a precise definition.
\begin{appendixthm}[Sobolev Approximation]
\label{thm:app_approx_para}
Let $a\in\R$, $b\in (a, \infty)$, and for every $d\in\N$ let 
\begin{equation*}
    \bar{u}_d(\gamma_{\sigma}  ,x,t) = \|x\|^2 + t \operatorname{Trace}(\gamma_\sigma \gamma_\sigma^{*}), \quad (\gamma_\sigma, x, t) \in [a,b]^{d\times d} \times [v,w]^d \times [0, T],
\end{equation*}
be the parametric solution map for the $d$-dimensional heat equation with paraboloid initial condition. Then there exists a constant $C>0$ with the following property: For every $\eps\in (0,1/2)$, $d\in\N$ there exists a ReLU network $\Phi_{\eps,d}$ with at most $\lfloor C d^4\log(d/\eps)\rfloor$ parameters satisfying that
\begin{equation*}
\|\Phi_{\eps,d} - \bar{u}_d\|_{W^{1,\infty}([a,b]^{d\times d}   \times [v,w]^d \times [0,T])}\le \eps.
\end{equation*}
\end{appendixthm}
\begin{proof}
Our result is based on ReLU network approximation results in~\cite[Propositions C.1 and C.2]{guhring2019error} and~\cite[Propositions III.2 and III.4]{Grohs2019DeepTheory}, which are extensions of the work by Yarotsky~\cite{yarotsky2017error}. Specifically, let $\Delta> 0$ and let $\operatorname{sq}\colon[-\Delta,\Delta] \to \R$ be the squaring function given by $\operatorname{sq}(x):=x^2$. Then there exists a ReLU network $\Phi_\eps^{sq}$ with $\mathcal{O}(\log(1/\eps))$ layers, $\mathcal{O}(1)$ neurons per layer, and parameters bounded by $\mathcal{O}(1)$ satisfying that
\begin{equation*}
\|\Phi_\eps^{sq} - \operatorname{sq} \|_{W^{1,\infty}([-\Delta,\Delta])} \le \eps.
\end{equation*}
By the polarization identity $xy=\tfrac{1}{2}((x+y)^2-x^2-y^2)$
an analogous result holds for the multiplication function $\operatorname{mult}\colon[-\Delta,\Delta]^2 \to \R$ given by $\operatorname{mult}(x,y) := xy$.
We can therefore imitate the representation
\begin{equation*}
\bar{u}_d(\gamma_{\sigma}  ,x,t) 
  = \sum_{i=1}^d \operatorname{sq}(x_i) +  \sum_{i,j=1}^{d} \operatorname{mult}\big(t, \operatorname{sq}((\gamma_\sigma)_{ij})\big)
\end{equation*}
using ReLU network concatenation and parallelization~\cite[Section 5]{elbrachter2018dnn}. Finally, we can estimate the error using a chain rule for ReLU networks, see~\cite{berner2019towards} and~\cite[Section B.1]    {guhring2019error}.
\end{proof}
Next, we show that our setting even allows for combined approximation and generalization results without curse of dimensionality. To prove this, we focus on the d-dimensional heat equation with varying diffusivity and Gaussian initial condition.
We first show that ReLU networks are capable of efficiently approximating the parametric solution map.
\begin{appendixthm}[Approximation] 
\label{thm:app_approx_gauss}
Let $a\in\R$, $b\in (a, \infty)$ and for every $d\in\N$ let 
\begin{equation}
\label{eq:heat_gauss_sol}
    \bar{u}_d(\gamma_{\sigma}  ,x,t) = \frac{1}{(1 + 2t\gamma_{\sigma}^2)^{d/2}}e^{-\tfrac{\|x\|^2}{1 + 2t\gamma_{\sigma}^2}}, \quad (\gamma_\sigma, x, t) \in [a,b] \times [v,w]^d \times [0, T],
\end{equation}
be the parametric solution map of the $d$-dimensional heat equation with Gaussian initial condition. Then there exist a constant $C>0$ and a polynomial $q\colon\R\to\R$ with the following property: For every $\eps\in(0,1/2)$, $d\in\N$ there exists a ReLU network $\Phi_{\eps,d}$ with at most $\lfloor q(\log(d/\eps) )\rfloor$ layers, at most $\lfloor Cd \rfloor$ neurons per layer, and parameters bounded by $C$ satisfying that
\begin{equation*}
\|\Phi_{\eps,d} - \bar{u}_d\|_{\mathcal{L}^{\infty}([a,b]   \times [v,w]^d \times [0,T])}\le \eps.
\end{equation*}
\end{appendixthm}
\begin{proof}
  The proof is based on combining ReLU approximation results for Chebyshev polynomials (see~\cite[Lemma  A.6]{Grohs2019DeepTheory}) and the squaring and multiplication functions $\operatorname{sq}$, $\operatorname{mult}$ (see the proof of Theorem~\ref{thm:app_approx_para}). Specifically, for given $\Delta> 0$ we can approximate the functions
  \begin{equation*}
      [0,\Delta] \ni x \mapsto h(x) :=\sqrt{ \tfrac{1}{1+2x}} \quad \text{and} \quad [0, \Delta] \ni x\mapsto g(x):= e^{-x^2}
  \end{equation*}
   up to precision $\eps$ by ReLU networks with $\mathcal{O}(\operatorname{polylog}(1/\eps))$ layers, $\mathcal{O}(1)$ neurons per layer, and parameters bounded by $\mathcal{O}(1)$.
   Moreover, observe that
\begin{equation*}
    \bar{u}_d(\gamma_{\sigma}  ,x,t) =\prod_{i=1}^d \operatorname{mult}\Big(g\big( \operatorname{mult}(x_i,f(t,\gamma_{\sigma}))\big),  f(t,\gamma_{\sigma})\Big)
\end{equation*}
where
\begin{equation*}
f(t,\gamma_{\sigma}) := h(\operatorname{mult}(t,\operatorname{sq}(\gamma_{\sigma})))= \sqrt{\tfrac{1}{1 + 2t\gamma_{\sigma}^2}}.
\end{equation*}
We can imitate this representation using ReLU network concatenation and parallelization and hierarchical, pairwise multiplications for the tensor product, see~\cite[Section 5 and Proposition 6.4]{elbrachter2018dnn}. Finally, we can estimate the error via the mean value theorem.
\end{proof}
Now we show that the number of samples $s$ in~\eqref{eq:emp_learn}, needed to learn the parametric solution map $\bar{u}$, does not suffer from the curse of dimensionality, either. To satisfy boundedness assumptions commonly used in statistical learning theory, we restrict ourself to clipped ReLU networks, the output of which is assumed to be bounded by $1$. This can be achieved by composing each ReLU network with a simple clipping function, which itself can be represented as a small ReLU network~\cite[Section A.4]{berner2018analysis}. Note that this incorporates our prior knowledge that the parametric solution map of the heat equation with Gaussian initial condition in~\eqref{eq:heat_gauss_sol} satisfies $\|\bar{u}_d\|_{\mathcal{L}^{\infty}([a,b]   \times [v,w]^d \times [0,T])} \le 1$.
\begin{appendixthm}[Generalization]
\label{thm:app_gen}
Let $a\in\R$, $b\in (a, \infty)$ and for every $d\in\N$ let
\begin{equation*}
    V_d:=\operatorname{vol}([a,b]   \times [v,w]^d \times [0,T]) = T(b-a)(w-v)^d,
\end{equation*}
let $\bar{u}_d \colon [a,b] \times [v,w]^d \times [0, T] \mapsto [0,1]$ be the parametric solution map of the $d$-dimensional heat equation with Gaussian initial condition as defined in~\eqref{eq:heat_gauss_sol}, let
\begin{equation*}
    \Lambda_{d} = (\Gamma_d, X_d, \mathcal{T}_d) \sim \mathcal{U}([a,b]   \times [v,w]^d \times [0,T])\quad \text{and} \quad N_d \sim\mathcal{N}(0,I_d)
\end{equation*}
be independent random variables, define $ Y_{d} = e^{-\|X_d +  \sqrt{\mathcal{T}_d} \, \Gamma_d N_d \|^2}$, and let $((\Lambda_{d,i}, Y_{d,i}))_{i\in\N}$ be i.i.d.\@ random variables with $(\Lambda_{d,1}, Y_{d,1}) \sim (\Lambda_d, Y_d)$.
Then there exist a constant $C>0$ and a polynomial $q\colon\R\to\R$ with the following property:
For every $\eps, \rho \in (0,1/2)$, $d,s\in\N$ with 
\begin{equation*}
    s \ge (d/\eps)^2 q(\log(d/\eps)) \log(1/\rho)
\end{equation*}
there exists a neural network architecture $\mathcal{A}_{\eps,d}$ with at most $\lfloor q(\log(d/\eps))\rfloor$ layers and at most $\lfloor Cd \rfloor$ neurons per layer such that every measurable empirical risk minimizer
\begin{equation*}
   \hat{\Phi}_{\eps,d,s}\colon\Omega \to \mathcal{H}_{\eps,d}, \quad \hat{\Phi}_{\eps,d,s}(\omega) \in \argmin_{\Phi\in\mathcal{H}_{\eps,d}} \tfrac{1}{s}\sum_{i=1}^s(\Phi(\Lambda_{d,i}(\omega)) - Y_{d,i}(\omega))^2, \quad \omega \in\Omega,
\end{equation*}
in a hypothesis space $\mathcal{H}_{\eps,d}$ of clipped ReLU networks with architecture $\mathcal{A}_{\eps,d}$ and parameters bounded by $C$ satisfies that
\begin{equation*}
    \P\Big[\tfrac{1}{V_d}\|\hat{\Phi}_{\eps,d,s} - \bar{u}_d\|^2_{\mathcal{L}^2([a,b]   \times [v,w]^d \times [0,T])} \le \eps \Big] \ge 1-\rho.
\end{equation*}
\end{appendixthm}
\begin{proof}
Let $\mathcal{A}_{\eps,d}$ be the architecture of the ReLU network $\Phi_{\eps/2,d}$ in Theorem~\ref{thm:app_approx_gauss}. To simplify notation, we define $\|{\cdot}\|_{\mathcal{L}^2}:=\|{\cdot}\|_{\mathcal{L}^2([a,b]   \times [v,w]^d \times [0,T])}$ and for every $\Phi\in\mathcal{H}_{\eps,d}$ we define its risk $\mathcal{R}(\Phi)$ and its empirical risk $\hat{\mathcal{R}}(\Phi)$ by
\begin{equation*}
    \mathcal{R}(\Phi) := \E\Big[\big(\Phi(\Lambda_d) - Y_d\big)^2 \Big] \quad \text{and} \quad \hat{\mathcal{R}}(\Phi) := \tfrac{1}{s}\sum_{i=1}^s(\Phi(\Lambda_{d,i}) - Y_{d,i})^2.
\end{equation*}
The fact that the regression function coincides with the parametric solution map (see Theorem~\ref{thm:app_learn_prob}) and the bias-variance decomposition (see~\cite[Lemma 2.2]{berner2018analysis} and~\cite{cucker2002mathematical}) imply that
\begin{equation*}
 \tfrac{1}{V_d}\|\hat{\Phi}_{\eps,d,s} - \bar{u}_d\|^2_{\mathcal{L}^2} = \underbrace{\vphantom{\tfrac{1}{V_d}}\mathcal{R}(\hat{\Phi}_{\eps,d,s})-\mathcal{R}(\Phi^*)}_{\textnormal{generalization error}} +
\underbrace{ \tfrac{1}{V_d}\|\Phi^*-\bar{u}_d\|^2_{\mathcal{L}^2}}_{\textnormal{approximation error}}
\end{equation*} 
where $\Phi^* \in \argmin_{\Phi\in\mathcal{H}_{\eps,d}}\|\Phi-\bar{u}_d\|_{\mathcal{L}^2}$ is a best approximation of $\bar{u}_d$ in $\mathcal{H}_{\eps,d}$.
Our choice of $\mathcal{A}_{\eps,d}$ and Theorem~\ref{thm:app_approx_gauss} ensure that
\begin{equation*}
   \tfrac{1}{V_d}\|\Phi^*-\bar{u}_d\|^2_{\mathcal{L}^2} \le \|\Phi^* - \bar{u}_d\|^2_{\mathcal{L}^{\infty}([a,b]   \times [v,w]^d \times [0,T])}\le \eps/2.
\end{equation*}
For the generalization error we make use of results on the covering numbers of neural network hypothesis spaces, see e.g.\@~\cite[Proposition 2.8]{berner2018analysis}.
They ensure the existence of clipped ReLU networks $(\Phi_i)_{i=1}^n\subset \mathcal{H}_{\eps,d}$ with
\begin{equation}
\label{eq:cov}
    \log(n)\in \mathcal{O} (d^2\operatorname{polylog}(d/\eps))
\end{equation}
such that balls of radius $\eps/64$
(w.r.t.\@ the uniform norm) around those functions cover $\mathcal{H}_{\eps,d}$.
Further, note that the boundedness of the target variable, i.e.\@ $\sup_{\omega\in\Omega}|Y_d(\omega)|\le 1$, and the boundedness of the clipped ReLU networks in our hypothesis space, i.e.\@ $\sup_{\Phi\in\mathcal{H}_{\eps,d}}\|\Phi\|_{\mathcal{L}^{\infty}([a,b]   \times [v,w]^d \times [0,T])}\le 1$, ensure that the (empirical) risk is (uniformly) Lipschitz continuous with
\begin{equation*}
    \operatorname{Lip}(\mathcal{R}) \le 4 \quad \text{and} \quad  \operatorname{Lip}(\hat{\mathcal{R}})\le 4,
\end{equation*}
see~\cite[Proof of Theorem 2.4]{berner2018analysis}.
Thus we can bound the generalization error by 
\begin{equation*}
\begin{split}
   \mathcal{R}(\hat{\Phi}_{\eps,d,s})-\mathcal{R}(\Phi^*) &\le \mathcal{R}(\hat{\Phi}_{\eps,d,s})-\hat{\mathcal{R}}(\hat{\Phi}_{\eps,d,s})+\hat{\mathcal{R}}(\Phi^*)-\mathcal{R}(\Phi^*) \\
   &\le 2\max_{i=1}^n \big|\mathcal{R}(\Phi_i)-\hat{\mathcal{R}}(\Phi_i) \big| + \tfrac{ 2\eps(\operatorname{Lip}(\mathcal{R}) +\operatorname{Lip}(\hat{\mathcal{R}}))}{64} \\
   & \le 2\max_{i=1}^n \big|\mathcal{R}(\Phi_i)-\hat{\mathcal{R}}(\Phi_i) \big| + \eps/4.
\end{split}
\end{equation*}
Employing Hoeffding's inequality~\cite{hoeffding1963} and a union bound, it holds that
\begin{equation*}
   \P\big[\max_{i=1}^n \big|\mathcal{R}(\Phi_i)-\hat{\mathcal{R}}(\Phi_i) \big| \le \eps/8 \big]   \ge 1-\rho
\end{equation*}
where we need $s\sim \log(n/\rho)/\varepsilon^2$ many samples.
Together with~\eqref{eq:cov} this implies the claim.
\end{proof}
 
\subsection{Implementation Details} \label{sec:impl}

First, we want to present a rigorous definition of our Multilevel network architecture.
\begin{definition}[Multilevel Architecture] \label{def:multi_lvl}
Let $L,q,p\in\mathbb{N}$, $\chi\in\{ 0,1\}$, and $\varrho\colon\mathbb{R}\to\mathbb{R}$. We define the Multilevel network $\Phi\colon \R^p\to\R$ with input dimension $\text{dim}_{\text{in}}(\Phi)=p$, $L$ levels, amplifying factor $q$, (component-wise applied) activation function $\varrho$, and residual constant $\chi$ for every $x\in\mathbb{R}^p$ by
\begin{equation} \label{eq:def_multi}
    \Phi(x):=\sum_{l=0}^{L-1} \Phi^{2^l}_l(x) \in\mathbb{R}
\end{equation}
where the intermediate network outputs $\Phi^{i}_l(x)$ are given by the following scheme:
\begin{align*}
    \Phi_l^1(x) &= \mathcal{A}_l^1(\varrho(\operatorname{Norm}_l^1(
    \mathcal{A}_l^{0}(x))), &l\in\{0,\dots,L-1\}, \\
    \Phi_l^{i}(x) &= \mathcal{A}_l^{i}(\varrho\operatorname{Norm}^i_l(\Phi_l^{i-1}(x) + \chi\Phi_{l+1}^{2i-2}(x))), &l\in\{1,\dots,L-2\}, \ i\in\{2,\dots,2^{l}\}, \\
    \Phi_{L-1}^i(x) &= \mathcal{A}_{L-1}^i(\varrho(\operatorname{Norm}_{L-1}^i(
    \Phi_{L-1}^{i-1}(x))), &i\in\{2,\dots,2^{L-1}\}.
\end{align*}
In the above, the constant $\chi$ controls whether we use intermediate residual connections, and for every $l\in\{0,\dots,L-1\}$ the functions 
\begin{equation*}
    \operatorname{Norm}^i_l\colon \mathbb{R}^{qp} \to \mathbb{R}^{qp}, \quad i\in \{ 1, \dots, 2^l \}, 
\end{equation*}
are denoting normalization layers, e.g.\@ batch normalization~\cite{ioffe2015batch} or layer normalization~\cite{ba2016layer}, and
\begin{equation*}
    \mathcal{A}_l^{0}\colon\mathbb{R}^p\to\mathbb{R}^{qp}, \quad
    \mathcal{A}^i_l\colon\mathbb{R}^{qp}\to\mathbb{R}^{qp}, \quad i\in \{ 1, \dots, 2^l-1 \}, \quad
    \mathcal{A}_l^{2^l}\colon\mathbb{R}^{qp}\to\mathbb{R}
\end{equation*} 
are learnable linear mappings (or affine-linear in case of $\mathcal{A}_l^{2^l}$).
\end{definition}
In the implementation of our examples we used $\chi=1$ to propagate intermediate residuals from the corresponding higher level using additive skip-connections, followed by a batch normalization layer as proposed by~\cite{ioffe2015batch}. This allows the length of the shortest gradient path during backpropagation to scale like the number of levels $L$ instead of the number of layers $2^L$; a feature commonly known to prevent diminishing or exploding gradients~\cite{yu2018deep}. Thus, we can maintain computational tractability while at the same time having rather deep architectures. Note that a certain depth is needed for our approximation and generalization results in Section~\ref{sec:theory}, as well as to optimally approximate certain families of functions \cite{montufar2014number, petersen2018optimal, yarotsky2017error}. We pick the ReLU activation function as non-linearity to remain consistent with our theoretical guarantees in Section~\ref{sec:theory} and with the growing body of literature on the approximation and generalization capabilities of ReLU networks. To optimize the networks we use the Adam optimizer (with decoupled weight decay regularization as proposed by~\cite{loshchilov2017decoupled}) and exponentially decaying learning rate. The precise setup is summarized in Table~\ref{table:hp} and the hyperparameters over which we optimized using Tune~\cite{li2018massively, liaw2018tune} are given in Table~\ref{table:opt}.
\newpage

\begin{enumerate}
\item \textbf{Input sets:} input sets for the parameter $\gamma = (\gamma_{\sigma}, \gamma_{\mu}, \gamma_{\varphi}) \in D_{\sigma} \times D_{\mu} \times D_{\varphi} = D$, the spatial variable $x\in [v,w]$, and the time variable $t\in [0,T]$, as defined in Section~\ref{section_affine_linear}.
    \item \textbf{Network:}
        input dimension $\text{dim}_{\text{in}}(\Phi)$,
        activation function $\varrho$,
        number of levels $L$, amplifying factor $q$, usage of intermediate residual connections $\chi$, normalization layers $\operatorname{Norm}_l^i$, and approximate number of parameters of the Multilevel architecture, see Definition~\ref{def:multi_lvl}.
     \item \textbf{Training:}
        computation of the SDE solution,
        optimizer,
        initialization of the linear mappings $\mathcal{A}_l^i$ where $\xi:= d_{\text{in}}^{-1/2}$ with $d_{\text{in}}$ denoting the input dimension,
         weight decay,
         batch-size,
        initial learning rate, and factor for learning rate decay each patience steps as long as the learning rate is larger than the minimal learning rate. Note that the training data size in~\eqref{eq:emp_learn} is given by $s=\text{batch-size} \cdot \text{\#steps}$ where the number of steps is reported in Tables~\ref{table:bs},~\ref{table:basket},~\ref{table:heat}, and~\ref{table:heat_gaussian}.
     \item \textbf{Validation:}
        pointwise computation of the PDE solution,
        batch-size, and
         number of batches per evaluation.\footnote{The evaluation of the PDE via Monte Carlo simulation as in~\eqref{eq:MC} is computationally very expensive. That is the reason why we only took one evaluation batch per iteration for the Basket put option. However, note that training the network with Euler-Maruyama simulated data does not increase the training time significantly (see Table~\ref{table:basket}) which underlines the general applicability of our algorithm.} Note that $n=\text{batch-size} \cdot \text{\#eval. batches}$ for each reported $\mathcal{L}^1 $-error, see~\eqref{eq:lp_err}.
     \item \textbf{Execution:}
        PyTorch module and random module seeds for the $4$ independent runs and
        number and type of GPUs per run.

\end{enumerate}
\begin{table}[t!]
    \centering
    \caption{Training setup}
    \begin{tabular}{lllll}
\toprule
  & \textbf{Black-Scholes} &  \textbf{Basket Put} & \textbf{Heat Paraboloid} & \textbf{Heat Gaussian} \\
\midrule 
\textbf{Input sets} & & & \\
$D_\sigma$ & $[0.1,0.6] \times \{0\}$ &  $([0.1,0.6]^{3\times 3})^4$  & $\{\vec{0}\}\times [0,1]^{10\times 10}$ & $\{\vec{0}\}\times [0,0.1]I_{150}$ \\ 
$D_\mu$ & $\{\vec{0}\}$ & $[0.1,0.6]^{3\times 4} $  & $\{\vec{0}\}$ & $\{\vec{0}\}$ \\
$D_\varphi$ &  $[10,12]  $  & $[10,12]$ & $\{\}$& $\{\}$ \\
$  [v,w] $ & $[9,10]$ & $[9,10] $ & $[0.5,1.5] $ & $[-0.1,0.1]$ \\
$   [0,T] $ & $ [0,1] $ & $ [0,1] $ & $[0,1]$ & $ [0,1]$ \\
\midrule
\textbf{Network} & & & \\
$\text{dim}_{\text{in}}(\Phi)$ & 4 & 53 & 111 & 152 \\ 
architecture & Multilevel & Multilevel & Multilevel& Multilevel \\
$(L,q,\chi)$ &         (4,5,1)  & (4,5,1) & (4,4,1) & (4,4,1)\\
activation $\varrho$ & ReLU & ReLU & ReLU& ReLU \\
$\operatorname{Norm}$ layer & batch norm. & batch norm. & batch norm. & batch norm.\\
\#parameters & 5.4K & 0.8M & 
2.4M & 
4.5M \\
\midrule
\textbf{Training} & & & \\
solution SDE & analytic & Euler-M. & analytic  & analytic\\
optimizer & AdamW & AdamW & AdamW & AdamW \\
param. init. & $\mathcal{U}([-\xi,\xi])$&$\mathcal{U}([-\xi,\xi])$ & $\mathcal{U}([-\xi,\xi])$ & $\mathcal{U}([-\xi,\xi])$\\
weight decay & $0.01$ & $0.01$ & $0.01$ & $0.01$ \\
batch-size &        $2^{16}$        &  $2^{17}$  & $2^{17}$ & $2^{17}$ \\
(init. lr., min. lr.) &          $(10^{-2}, 10^{-8})$   &      $(10^{-3}, 10^{-8})$ & $(10^{-3}, 10^{-8})$ & $(10^{-3}, 10^{-8})$\\
(decay, patience) &          $(0.25, 4000)$   &      $(0.4, 4000)$ & $(0.4, 4000)$ & $(0.4, 4000)$ \\
\midrule
\textbf{Validation} & & & \\
solution PDE & analytic & MC-approx. & analytic & analytic \\
batch-size &        $2^{16}$         &  $2^{17}$ & $2^{17}$ & $2^{17}$ \\
\#eval. batches & 150 & 1 & 150 & 150 \\
\midrule
\textbf{Execution} & & & \\
seeds & 0,1,2,3 & 0,1,2,3 & 0,1,2,3 & 0,1,2,3 \\
\#GPUs per run & 2 (Tesla V100) & 4 (Tesla V100) & 2 (Tesla V100) & 2 (Tesla V100) \\
\bottomrule
\end{tabular}
    \label{table:hp}
\end{table}
\begin{table}[!tb]
\centering
\caption{Ranges for hyperparameter optimization}
\begin{tabular}{llll}
\toprule
  hyperparameter & range \\
\midrule
  $(L,q)$ &       $\{3,4\} \times \{4,5,6\}$  \\
 optimizer & $\{$AdamW, SGD (with momentum \& weight decay)$\}$ \\
 batch-size &        $\{16384, 32768, 65536, 131072 \}$ \\
 learning rate &          $(10^{-1},10^{-5})$  \\
 lr. decay factor  & $(0.2,0.6)$ \\
\bottomrule
\end{tabular}
    \label{table:opt}
\end{table}
\subsection{Additional Numerical Results} \label{sec:add_num}

In Tables~\ref{tab:abl_bs} and~\ref{table:abl_heat} we present an ablation study which empirically proves the superior performance of our Multilevel architecture in combination with batch normalization compared to feed-forward architectures or the usage of layer normalization~\cite{ba2016layer}. For the feed-forward architecture we used the network $\Phi^{2^L}_L$ defined in~\eqref{eq:def_multi}, i.e.\@ only the highest level of the corresponding Multilevel network with $L+1$ levels and $\chi=0$. Despite having slightly less parameters, our Multilevel architecture consistently outperforms the feed-forward architecture. Moreover, the use of residual connections, i.e.\@ $\chi=1$, has a positive impact. Note that all not-mentioned settings are kept as in Table~\ref{table:hp}.

The performance of our algorithm in the case of the Black-Scholes option pricing model from Section~\ref{section_black_scholes} is further illustrated in Figures~\ref{fig:bs2},~\ref{fig:greek2},~\ref{fig:error_bs2}, and~\ref{fig:error_greek2}. Finally, Figure~\ref{fig:heat_dims} depicts the computational cost of our algorithm as a function of the problem input dimension for the heat equation with paraboloid initial condition.

\begin{table}[!tb]
        \centering
        \caption{Ablation study for the Black-Scholes model}
\begin{tabular}{llll}
\toprule
architecture, normalization layer & avg. time (s)   & avg. best $\mathcal{L}^1$-error & \#parameters \\
\midrule
Feed-Forward, layer norm. &       \phantom{0}809 $\pm$ 9 &    0.1476 $\pm$ 0.0772 &        6741 \\
Feed-Forward,  none  &        \phantom{0}496 $\pm$ 26 &    0.0526 $\pm$ 0.0002 &        6101 \\
Feed-Forward, batch norm. &      3755 $\pm$ 57 &    0.0017 $\pm$ 0.0003 &        6741 \\
Multilevel $\chi=0$, layer norm. &       \phantom{0}867 $\pm$ 10 &       0.0349 $\pm$ 0.0000 &        5404 \\
Multilevel $\chi=0$, none &        \phantom{0}570 $\pm$ 6 &    0.0069 $\pm$ 0.0001 &        4804 \\
Multilevel $\chi=0$, batch norm. &      3414 $\pm$ 18 &       0.0012 $\pm$ 0.0000 &        5404 \\
Multilevel $\chi=1$, layer norm. &        \phantom{0}874 $\pm$ 13 &    0.0348 $\pm$ 0.0001 &        5404 \\
Multilevel $\chi=1$, none  &       \phantom{0}581 $\pm$ 10 &       0.0069 $\pm$ 0.0000 &        4804 \\
Multilevel $\chi=1$, batch norm.  &      3453 $\pm$ 34 &    \textbf{0.0011} $\pm$ 0.0001 &        5404 \\
\bottomrule
\end{tabular}
\label{tab:abl_bs}
\end{table}

\begin{table}[!tb]
        \centering
        \caption{Ablation study for the heat equation with paraboloid initial condition}
\begin{tabular}{llll}
\toprule
architecture & avg. time (s)  & avg. best $\mathcal{L}^1$-error & \#parameters \\
\midrule
Feed-Forward &     14764 $\pm$ 65 &     0.0090 $\pm$ 0.0003 &     3020977 \\
Multilevel $\chi=0$ &     13892 $\pm$ 83 &    0.0058 $\pm$ 0.0001 &     2380732 \\
Multilevel $\chi=1$        &      14049 $\pm$ 138 &    \textbf{0.0055} $\pm$ 0.0001 &     2380732 \\
\bottomrule
\end{tabular}
    \label{table:abl_heat}
\end{table}

\begin{figure}[!tb]
    \centering
    \begin{minipage}[b]{.45\textwidth}
        \centering
        \includegraphics[width=\linewidth]{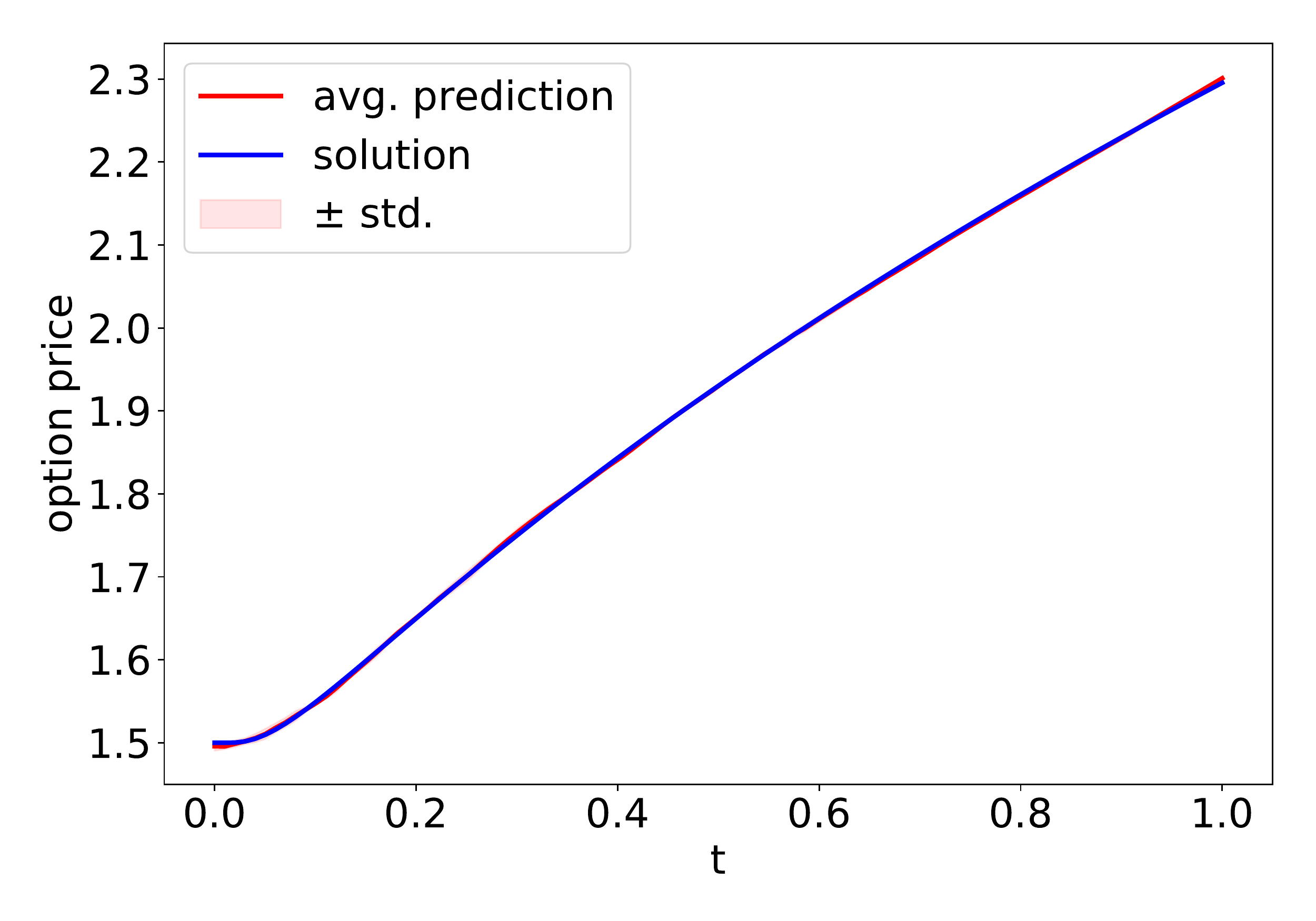}
        \caption{Shows $\bar{u}(\gamma, x, \cdot)$ vs.\@ the average prediction (and its standard deviation) at $x=9.5$, $\gamma_\sigma=0.35$, and $\gamma_\varphi=11$.}
        \label{fig:bs2}
    \end{minipage}
    \hspace{2em}
    \begin{minipage}[b]{0.45\textwidth}
        \centering
        \includegraphics[width=\linewidth]{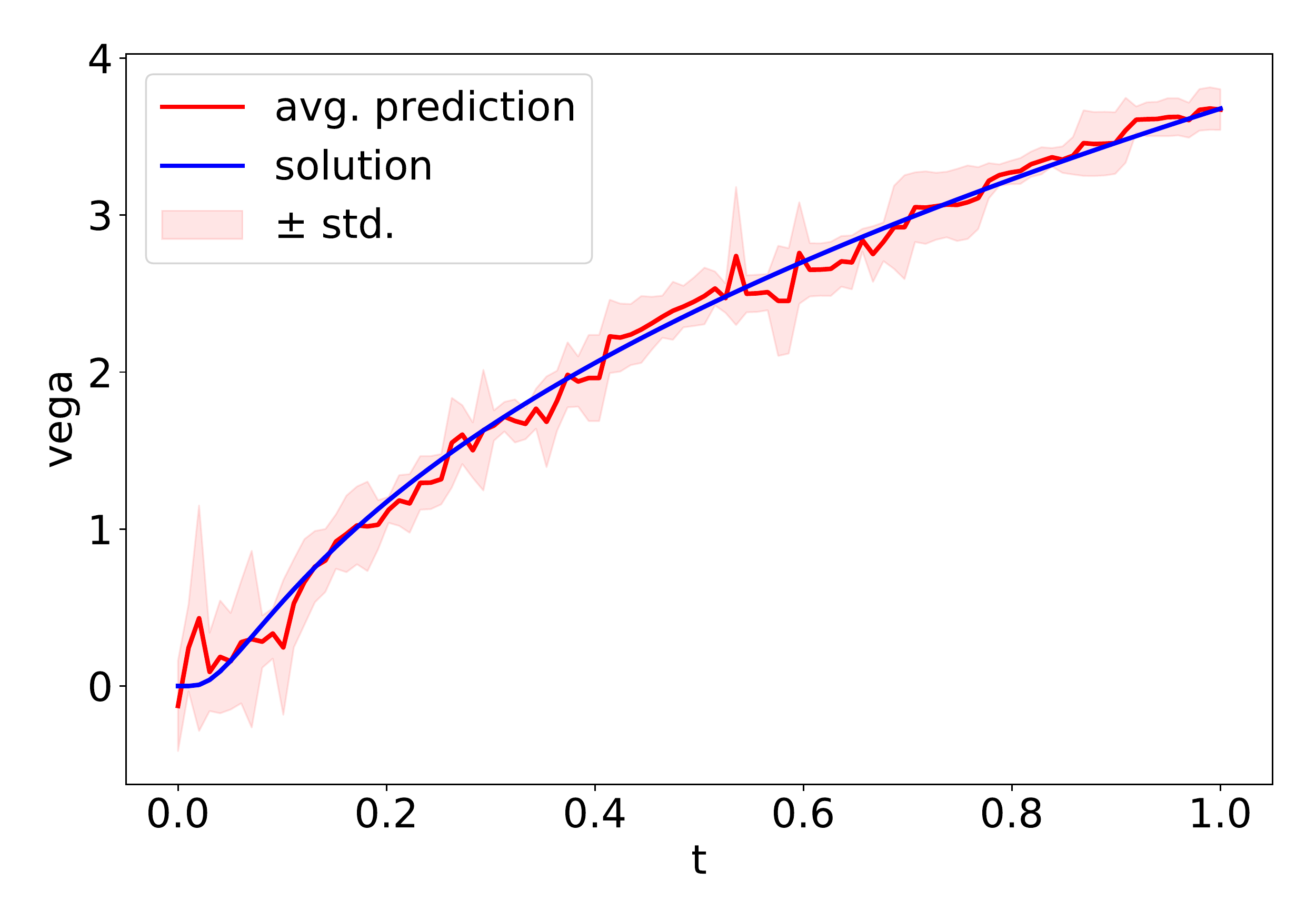}
        \caption{Shows the Vega $\frac{\partial \bar{u}}{\partial \gamma_\sigma}(\gamma, x, \cdot)$ vs.\@ the average prediction (and its standard deviation) at $x=9.5$, $\gamma_\sigma=0.35$, and $\gamma_\varphi=11$.}
        \label{fig:greek2}
    \end{minipage}
\end{figure}
\begin{figure}[!tb]
    \centering
    \begin{minipage}[b]{0.45\textwidth}
        \centering
        \includegraphics[width=\linewidth]{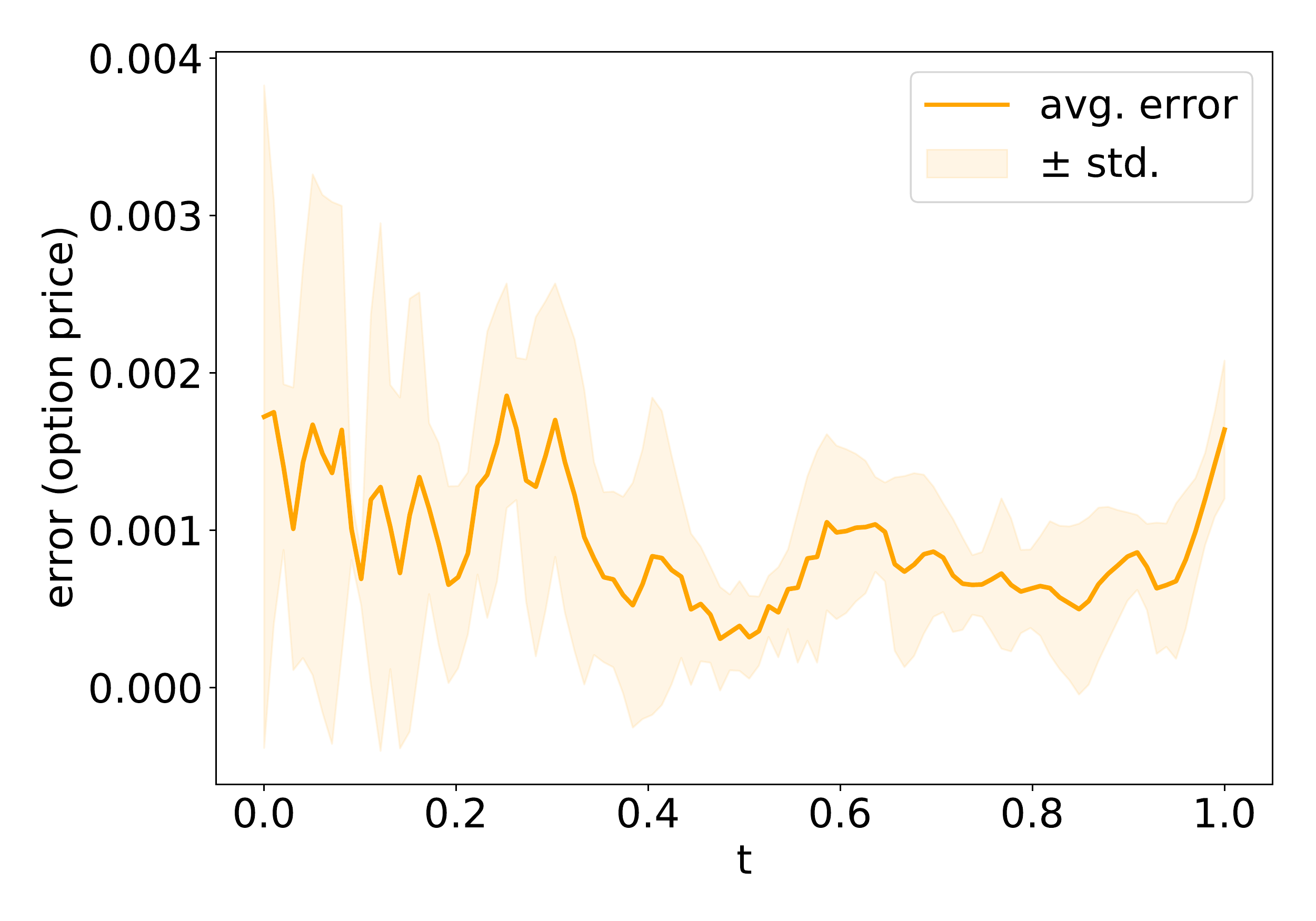}
        \caption{Shows the average prediction error $\tfrac{\vphantom{\frac{\partial \Phi}{\partial \gamma_\sigma}}|\Phi(\gamma, x, {\cdot}) -\bar{u}(\gamma, x, {\cdot})|}{\vphantom{\frac{\partial \Phi}{\partial \gamma_\sigma}}1+|\bar{u}(\gamma, x, {\cdot})|}$ and its standard deviation at $x=9.5$, $\gamma_\sigma=0.35$, and $\gamma_\varphi=11$.}
        \label{fig:error_bs2}
    \end{minipage}
    \hspace{2em}
    \begin{minipage}[b]{0.45\textwidth}
        \centering
        \includegraphics[width=\linewidth]{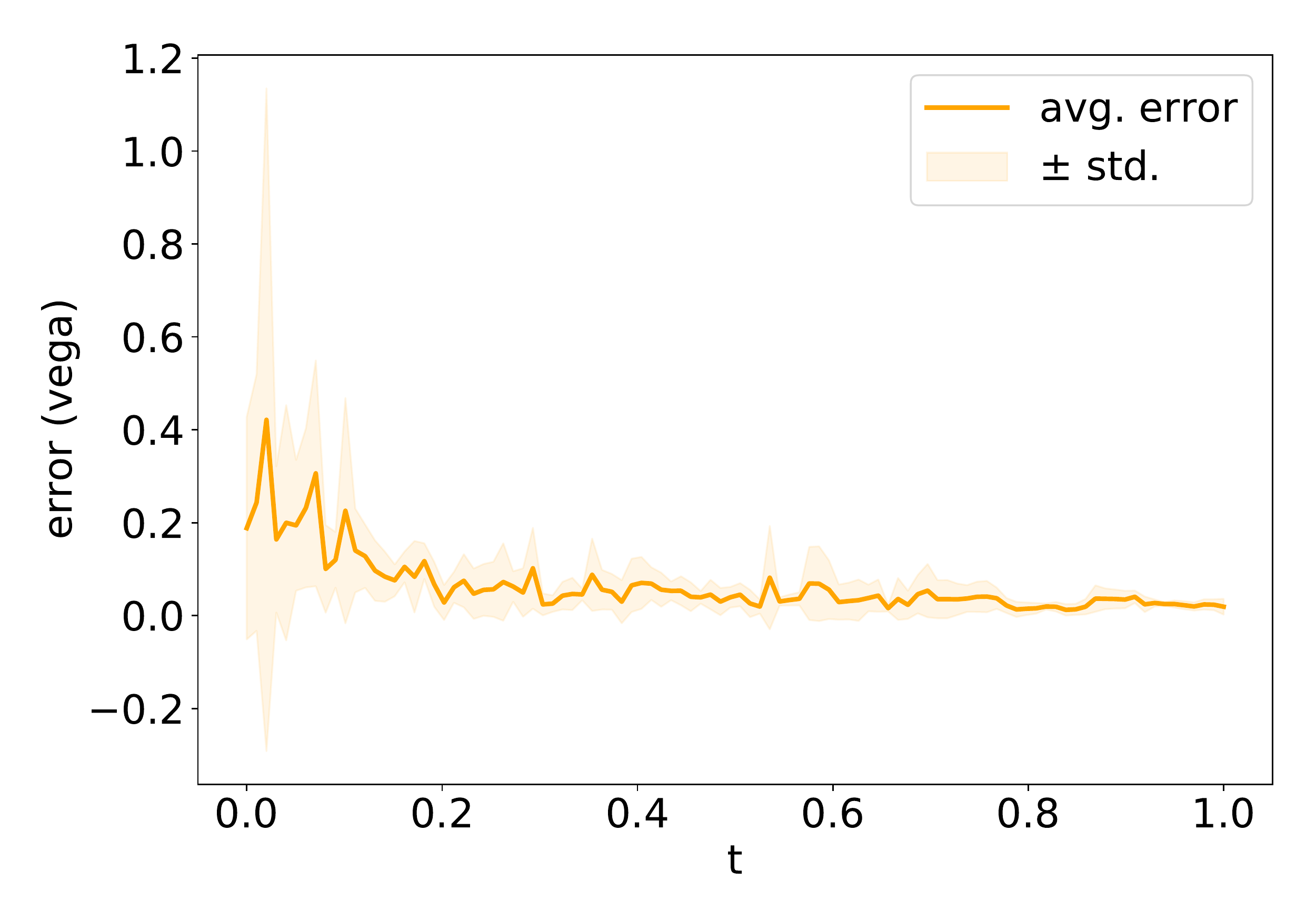}
        \caption{Shows the average error of the Vega $\tfrac{|\frac{\partial \Phi}{\partial \gamma_\sigma}
        (\gamma, x, {\cdot}) -\frac{\partial \bar{u}}{\partial \gamma_\sigma}(\gamma, x, {\cdot})|}{1+|\frac{\partial \bar{u}}{\partial \gamma_\sigma}(\gamma, x, {\cdot})|}$ and its standard deviation at $x=9.5$, $\gamma_\sigma=0.35$, and $\gamma_\varphi=11$.}
        \label{fig:error_greek2}
    \end{minipage}
\end{figure}
\begin{figure}[!tb]
        \centering
        \includegraphics[width=0.90\linewidth]{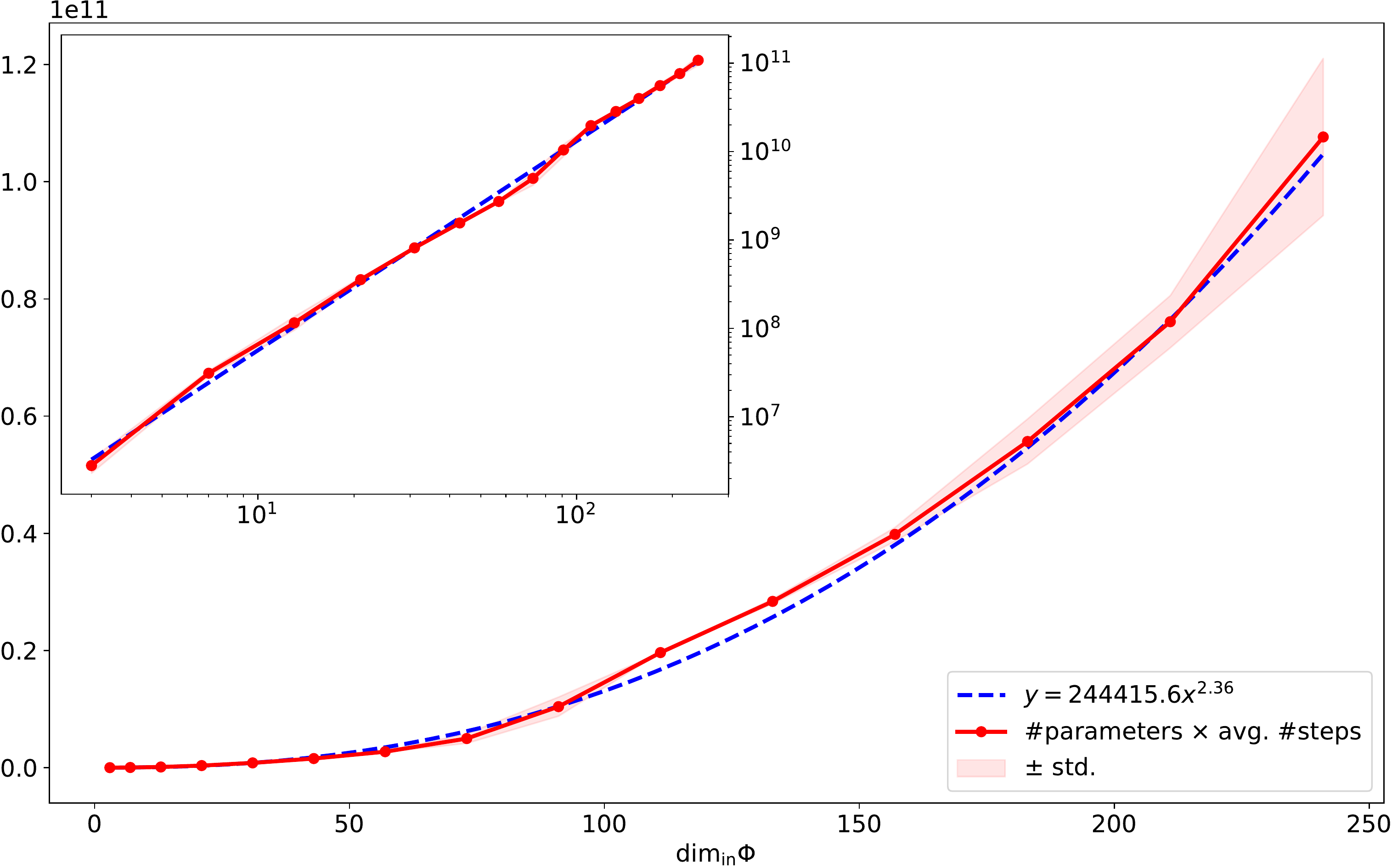}
        \caption{Shows the cost in terms of number of network parameters times average number of steps to achieve an $\mathcal{L}^1$-error of $10^{-2}$ w.r.t.\@ to the problem dimension $d^2+d+1$ for the heat equations with paraboloid initial condition and $d=1,\dots,15$. 
        The absence of the curse of dimensionality is underlined by the linear behaviour in the log-log inset. The error was evaluated every $250$ gradient descent steps and except of the varying dimension all settings are kept as in Table~\ref{table:hp}.}
        \label{fig:heat_dims}
\end{figure}

\end{document}